\documentclass{article}

\usepackage[preprint]{neurips_2026}

\usepackage[utf8]{inputenc} 
\usepackage[T1]{fontenc}    
\usepackage{hyperref}       
\usepackage{url}            
\usepackage{booktabs}       
\usepackage{amsfonts}       
\usepackage{nicefrac}       
\usepackage{microtype}      
\usepackage{xcolor}         

\RequirePackage{fancyhdr}
\RequirePackage{algorithm}
\RequirePackage{algorithmic}
\RequirePackage{eso-pic} 
\RequirePackage{forloop}
\usepackage{natbib}
\usepackage{amssymb}
\usepackage{mathtools}
\usepackage{amsthm}
\theoremstyle{plain}
\newtheorem{theorem}{Theorem}[section]
\newtheorem{proposition}[theorem]{Proposition}
\newtheorem{lemma}[theorem]{Lemma}
\newtheorem{corollary}[theorem]{Corollary}
\theoremstyle{definition}

\theoremstyle{remark}

\usepackage{cleveref}
\crefname{equation}{Eq.}{Eqs.}
\crefname{table}{Table}{Tables}
\crefname{figure}{Figure}{Figures}
\crefname{section}{Section}{Sections}
\crefname{algorithm}{Algorithm}{Algorithms}
\crefname{proposition}{Proposition}{Propositions}
\crefname{corollary}{Corollary}{Corollaries}
\crefname{lemma}{Lemma}{Lemmas}
\crefname{definition}{Definition}{Definitions}
\crefname{assumption}{Assumption}{Assumptions}
\usepackage{xcolor,colortbl}
\usepackage{adjustbox}
\usepackage{url}
\usepackage{multirow,multicol,xspace}
\usepackage{float}
\usepackage{graphics}
\usepackage{paralist}
\usepackage{wrapfig}
\usepackage[normalem]{ulem}
\usepackage{multirow}
\usepackage{adjustbox}
\usepackage{threeparttable}
\usepackage[font=small,skip=6pt]{caption}

\newcommand{\MODEL}{CoMole\xspace}
\newcommand{\MODELSFT}{\ensuremath{\textnormal{CoMole}_{\text{w/o RL}}}}
\newcommand{\BEST}[1]{\textbf{\textcolor{purple}{#1}}}

\title{Controllable Molecular Generative Foundation Models}

\author{%
    Yihan Zhu\\
    University of Notre Dame\\
    \texttt{yzhu25@nd.edu}\\
    \And 
    Yuhan Liu\\
    University of Notre Dame\\
    \texttt{yliu57@nd.edu}\\
    \And 
    Weijiang Li\\
    University of Notre Dame\\
    \texttt{wli27@nd.edu}\\
    \And 
    Tengfei Luo\\
    University of Notre Dame\\
    \texttt{tluo@nd.edu}\\
    \And 
    Meng Jiang\\
    University of Notre Dame\\
    \texttt{mjiang2@nd.edu}\\
}

\begin{document}

\maketitle

\begin{abstract}
Despite the success of foundation models in language and vision, molecular graph generation still lacks a unified framework for heterogeneous design tasks with reliable controllability. 
While reinforcement learning (RL) offers a natural post-training mechanism for task-specific optimization, applying it to graph generative models is hindered by the vast atom-wise action spaces and chemically invalid intermediate states.
We propose \textbf{Co}ntrollable \textbf{Mole}cular Generative Foundation Models (\MODEL), built with a unified motif-aware graph diffusion pipeline. 
By learning a motif-aware graph space, \MODEL transfers pretrained structural priors into controllable generation, where RL optimizes conditional reverse policies over chemically meaningful decisions.
We theoretically characterize the bottleneck of atom-level RL and justify motif-aware policy optimization.
Across three heterogeneous benchmarks spanning materials and drug discovery, \MODEL ranks first in controllability on all nine targets, reduces MAE by up to 48.2\% relative to the strongest baselines, and maintains validity above 0.94 without rule-based correction or post-hoc filtering. 
We further show that \MODEL transfers controllability to unseen properties by optimizing only task embeddings with the generator frozen, achieving performance competitive with strong task-specific baselines.
\end{abstract}

\section{Introduction}
\label{sec:introduction}
Molecular inverse design, generating structures with desired functional properties, is a central challenge in scientific discovery, with applications spanning biomedicine and materials science.
Recent graph diffusion methods have significantly advanced molecular generation~\citep{vignac2023digress,huang2023cdgs,liu2024graphdit}. However, these approaches remain largely task-specific and offer limited controllability over target properties~\citep{qin2025defog,liu2025demodiff}.
In language and vision, foundation-model paradigms have enabled powerful generative systems through large-scale pretraining (PT), supervised fine-tuning (SFT), and reinforcement learning (RL)-based alignment~\citep{bommasani2022opportunitiesrisksfoundationmodels,zhang2023buildinggeneralfoundationmodels}. 
For molecular inverse design, the abundance of unlabeled chemical data alongside label-scarce downstream tasks motivates the development of unified molecular generative foundation models to bridge the controllability gap between current methods and practical inverse-design needs.

However, instantiating this paradigm with atom-level graph diffusion exposes a core bottleneck: beyond the trajectory collapse observed in atom-level diffusion RL~\citep{dulac2015deeprl,liu2024gdpo}, the vast space of low-level atom-wise actions is poorly aligned with the chemically meaningful substructures through which chemists typically formulate structure-property rationales for molecular design.
As illustrated in \cref{fig:main}, atom-level RL often fails to construct reliable substructures from local edits before it can associate structural changes with property rewards.
This fragility arises because each action must jointly coordinate atom types, bonds, and valence constraints. A single invalid choice can push the trajectory off the feasible chemical manifold and make later denoising difficult to recover.
For example, the atom-level variant of our design collapses rule-free validity from 0.95 to 0.07 and worsens gas-permeability control MAE by 2.68$\times$ (\cref{tab:rq2_motif_rl}).

To overcome this bottleneck, we reparameterize graph-diffusion RL with a motif-aware decision space.
It preserves local structural flexibility while introducing rings, functional groups, and data-driven motifs as higher-level decisions.
Since molecular properties depend on substructures and their complex interactions, this space lifts RL from fragile atom-wise construction paths to chemically meaningful decisions (e.g., attaching a benzene ring), allowing rewards on generated structures to be credited more directly to property-relevant actions.
This abstraction also promotes chemical validity by preserving internally coherent substructures and reducing the risk of trajectory corruption from invalid atom-wise edits. Beyond using motifs mainly as generation priors~\citep{jin2018junction,kong2022psvae}, our design uses them to stabilize RL optimization toward controllable generation.

We introduce \textbf{\MODEL}, to our knowledge the first family of \textbf{Co}ntrollable \textbf{Mole}cular Generative Foundation Models for heterogeneous inverse design.
The central idea is to transfer pretrained structural knowledge into controllable inverse design by optimizing conditional reverse policies via RL over chemically meaningful decisions.
Concretely, \MODEL learns a Node Pair Encoding (NPE)-based tokenizer~\citep{liu2025demodiff} that preserves singleton atoms and attachment-level information while merging frequent adjacent units from the pretraining distribution into motif-aware graph states.
Over this space, the graph diffusion transformer is trained through three stages: PT learns transferable structural priors, SFT introduces conditioning for multiple properties, and RL aligns the conditional reverse diffusion policy with terminal target-property rewards.
In the main text, we instantiate RL alignment with proximal policy optimization (PPO), while alternative policy-optimization objectives are discussed in Appendix~\ref{app:alternative_rl_objective}.

We theoretically characterize the structural bottleneck of atom-level RL and motivate motif-aware policy optimization.
Empirically, we evaluate \MODEL on three materials and drug-discovery benchmarks spanning numerical and categorical conditions. 
Across all nine targets, \MODEL ranks first in controllability:
on both polymer benchmarks, it reduces MAE by over 44\% relative to the best baseline.
On small-molecule tasks, it reduces FreeSolv MAE by 13.1\% and achieves 1.0 accuracy on BACE classification.
These gains are achieved while maintaining validity above 0.94 without rule-based correction or post-hoc filtering.
We further show that \MODEL transfers controllability to unseen property targets by learning only task embeddings while keeping the generator frozen, achieving performance competitive with baselines trained directly on those targets. Together, these results suggest that our design learns transferable structure-property knowledge that supports controllable generation across heterogeneous inverse-design tasks.

\begin{figure}[tbp]
  \centering
  \includegraphics[width=0.95\linewidth]{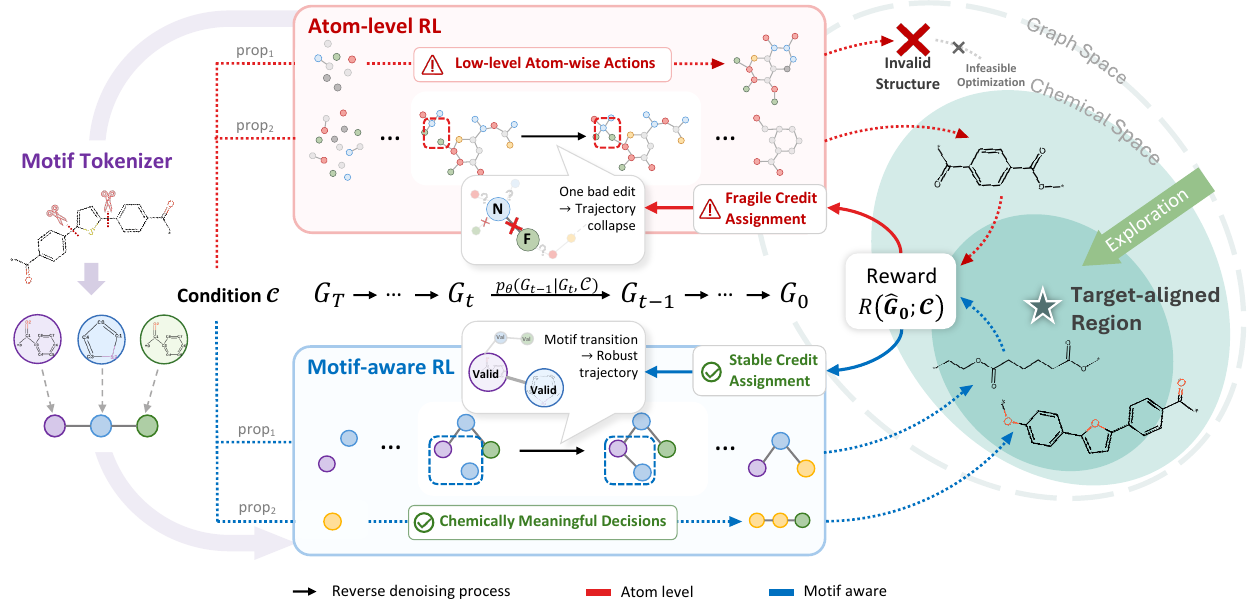}
    \caption{\textbf{Motif-aware RL as a key stage in training controllable molecular generative foundation models.}
    Atom-level RL over vast, low-level graph edits suffers trajectory collapse and fragile credit assignment,  whereas motif-aware RL credits terminal rewards to chemically meaningful decisions, stabilizing policy updates.}
  \label{fig:main}
\end{figure}

\section{Preliminaries}
\label{sec:preliminaries}
\paragraph{Notation.}
\label{sec:prelim_notation}
Let \(x\in\mathcal M_{\mathrm{val}}\) be a chemically valid molecular graph and \(\phi\) a learned graph tokenizer.
The tokenizer maps \(x\) to a motif graph \(z_0=\phi(x)\in\mathcal Z_0\), which serves as the motif state for diffusion generation and RL post-training.
For conditional generation, \(c=(k,y^\star)\in\mathcal C\) denotes the target condition, with task identity \(k\) and target value or label \(y^\star\).
We use \(\mathsf D\) to denote the empirical training distribution.

\subsection{Motif Graph Tokenization}
\label{sec:prelim_tokenization}
We follow the Node Pair Encoding (NPE) algorithm introduced by DemoDiff~\citep{liu2025demodiff} and learn a vocabulary from the pretraining dataset (Appendix~\ref{app:motif_tokenizer}).
The nodes of the resulting motif graph are atom-disjoint vocabulary units, including singleton atoms, preserved ring units, and data-driven merged substructures as higher-level motifs.
The graph also stores inter-motif bond labels and directional attachment-position labels for lossless reconstruction.
Since motif graphs have variable sizes, we pad them to a fixed number of motif slots and represent each state as
$
z=(X,E,P,m),
$
where \(m_i=1\) indicates that slot \(i\) is active.
The variable \(X_i\) encodes the motif type at slot \(i\).
For active motif pairs, \(E_{ij}=E_{ji}\) denotes a symmetric categorical bond label, including a no-bond category.
\(P_{ij}\) denotes the directional attachment-position label on motif \(i\) to motif \(j\). 
Directed pairs without an attachment use a null attachment-position label.
\(m\) is sampled once from a size prior and fixed throughout generation (see Appendix~\ref{app:mask_handling}).
When no confusion arises, we omit \(m\) and write \(z=(X,E,P)\).

\subsection{Molecular Design with Graph Diffusion Transformers}
\label{sec:prelim_diffusion}

Given a motif graph \(z_0=(X_0,E_0,P_0)\), we perform discrete diffusion over motif states \(z_t=(X_t,E_t,P_t)\). 
The forward noising process is
$
q(z_{1:T}\mid z_0)=\prod_{t=1}^{T} q(z_t\mid z_{t-1}),
$
which progressively corrupts motif types, bond types, and attachment-position labels.
Let \(q_t(z_t\mid z_0)\) denote the induced marginal at step \(t\). 
The reverse process starts from the prior \(p(z_T)\) and is parameterized by a graph diffusion transformer:
$
p_\theta(z_{0:T}\mid c)
=
p(z_T)\prod_{t=1}^{T}\widetilde p_\theta(z_{t-1}\mid z_t,t,c).
$
For unconditional pretraining, \(c\) is omitted.
At each reverse step, the denoiser predicts the distributions
$
(\hat X_0,\hat E_0,\hat P_0)=f_\theta(z_t,t,c),
$
which parameterize \(\widetilde p_\theta(z_{t-1}\mid z_t,t,c)\) by first predicting the motif state \(z_0=(X_0,E_0,P_0)\).
The model is trained with a masked denoising objective:
\(\mathcal L_{\mathrm{diff}}(\theta)
=
\mathbb E_{(z_0,c)\sim \mathsf D,\ t\sim \mathrm{Unif}([T]),\ z_t\sim q_t(\cdot\mid z_0)}
\left[
\lambda_X\mathrm{CE}_X+\lambda_E\mathrm{CE}_E+\lambda_P\mathrm{CE}_P
\right]\).
Here \(\mathrm{CE}_X\), \(\mathrm{CE}_E\), and \(\mathrm{CE}_P\) are masked cross-entropy losses over active motif nodes, inter-motif edges, and directed motif pairs. Details are given in Appendix~\ref{app:masked_ce}.

\section{Controllable Molecular Generative Foundation Modeling}
\label{sec:method}
Following \cref{sec:prelim_notation}, \MODEL generates molecules by reversing a task-conditioned motif-graph diffusion process. 
\cref{sec:method_policy} formulates PPO over the reverse trajectories, and \cref{sec:theory} analyzes why our motif-aware design improves over atom-level RL.

\subsection{Learning the Reverse Diffusion Process as a Policy}
\label{sec:method_policy}
Given a condition \(c\), reverse diffusion generates a graph through denoising \(z_T\) into \(z_0\). 
Since intermediate noisy graphs are difficult to evaluate reliably with property oracles, we apply rewards only on the terminal graph.
We therefore formulate reverse diffusion as a finite-horizon terminal-reward Markov decision process (MDP) and optimize the denoising policy over sampled trajectories.

\paragraph{MDP Formulation.}
We treat the reverse process as an MDP with horizon \(H=T\). 
Given a target condition \(c\), at MDP step \(h=0,\dots,T-1\), the state \(s_h=(z_{T-h},T-h,c)\) contains the current noisy motif graph, timestep, and condition, and the action is the next reverse state \(a_h=z_{T-h-1}\). 
During RL training, each rollout draws a condition \(c\sim\mu_{\mathrm{RL}}\) and samples \(z_T\sim p(z_T)\), starting from \(s_0=(z_T,T,c)\) and ending at \(s_T=(z_0,0,c)\).
After the policy samples \(a_h\), the next state is deterministically updated to \(s_{h+1}=(a_h,T-h-1,c)\), so the only stochastic decision is the one-step reverse kernel:
\begin{equation}
\label{eq:main_policy}
\pi_\theta(a_h\mid s_h)
=
\widetilde p_\theta(z_{T-h-1}\mid z_{T-h},T-h,c).
\end{equation}
Although \(a_h\) is a structured motif-graph action, 
\(\pi_\theta(a_h\mid s_h)\) factorizes over motif, bond, and attachment-position variables, 
allowing exact log-probability computation for RL.
The explicit factorization is given in Appendix~\ref{app:factorized_logprob}.

\paragraph{Terminal Molecular Reward.}
Let \(x_{\mathrm{gen}}=\mathrm{Dec}(z_0)\) be the molecule decoded from the final reverse state.
For task \(k\), let \(\hat o_k(x)\) be the oracle output and define the target discrepancy
\(d_c(x):=\ell_k(\hat o_k(x),y^\star)\) for valid molecules.
Here \(\ell_k\) is task-specific, e.g., absolute error for regression or absolute probability-label difference for binary classification.

We use a terminal reward that combines validity and target satisfaction:
\begin{equation}
\label{eq:reward_total_main}
R(z_0;c)
=
w_{\mathrm{val}}\,r_{\mathrm{val}}(x_{\mathrm{gen}})
+
(1-w_{\mathrm{val}})\,r_{\mathrm{prop}}(x_{\mathrm{gen}};c),
\qquad
w_{\mathrm{val}}\in[0,1].
\end{equation}
Here
\begin{equation}
\label{eq:reward_eq}
r_{\mathrm{val}}(x)
=
\begin{cases}
1, & x\in\mathcal M_{\mathrm{val}},\\
-1, & x\notin\mathcal M_{\mathrm{val}},
\end{cases}
\qquad
r_{\mathrm{prop}}(x;c)
=
\begin{cases}
g_k(d_c(x)), & x\in\mathcal M_{\mathrm{val}},\\
0, & x\notin\mathcal M_{\mathrm{val}}.
\end{cases}
\end{equation}
For regression tasks, we use
\(g_k(d)=\exp[-(d/\sigma_k)^2]\), with \(\sigma_k>0\).
For binary classification tasks, we set \(g_k(d)=1-d\), so \(r_{\mathrm{prop}}\) equals the oracle probability assigned to the target label.
Thus \(R(z_0;c)\in[-w_{\mathrm{val}},1]\).

\paragraph{Policy Optimization.}
Given the rollout distribution \(\mu_{\mathrm{RL}}\), the RL objective is
\begin{equation}
\label{eq:rl_objective}
J(\pi_\theta)
=
\mathbb E_{c\sim \mu_{\mathrm{RL}}}
\mathbb E_{\tau=(z_T,\ldots,z_0)\sim p_{\pi_\theta}(\cdot\mid c)}
\big[
R(z_0;c)
\big].
\end{equation}
We optimize \eqref{eq:rl_objective} with PPO initialized from the SFT checkpoint.
Because the reward is terminal-only and \(\gamma=1\), every reverse step in a rollout has return \(R(z_0;c)\).
At each MDP step \(h\), we use an MLP value head to estimate the state value \(V_\psi(s_h)\).
The advantage is \(A_h=R(z_0;c)-V_\psi(s_h)\).

For on-policy rollouts collected by \(\pi_{\theta_{\mathrm{old}}}\), the PPO importance ratio is
\begin{equation}
\label{eq:ratio_main}
\rho_h
=
\frac{\pi_\theta(a_h\mid s_h)}
{\pi_{\theta_{\mathrm{old}}}(a_h\mid s_h)}.
\end{equation}
We maximize the clipped PPO surrogate
\begin{equation}
\label{eq:ppo_clip_main}
\mathcal J_{\mathrm{clip}}(\theta)
=
\mathbb E\!\left[
\sum_{h=0}^{T-1}
\min\!\Big(
\rho_h A_h,\;
\mathrm{clip}(\rho_h,1-\epsilon,1+\epsilon)A_h
\Big)
\right],
\end{equation}
In implementation, we optimize the PPO loss with critic loss, entropy bonus, and KL regularization.
Details are given in Appendix~\ref{app:ppo_details}.


\subsection{Theoretical Analysis}
\label{sec:theory}
We analyze the RL stage and the motif-aware decision space.
Let \(\pi_{\mathrm{ref}}\) denote the frozen SFT reference policy used for KL regularization, \(\mathcal A(s)\) the finite reverse action space at state \(s\), and \(p_\pi(\tau\mid c)\) the trajectory distribution induced by policy \(\pi\) under condition \(c\).

\subsubsection{The Role of RL in Condition Control}
\label{sec:theory_rl_improves}
As in \cref{sec:method_policy}, reverse diffusion defines a finite-horizon MDP with terminal reward \(R(z_0;c)\) on the final clean motif graph \(z_0\).
Following KL-regularized control and control-as-inference formulations
\citep{todorov2006linearly,levine2018reinforcement}, for a fixed condition \(c\) we consider the population objective
\begin{equation}
\label{eq:regularized_control_objective}
\mathcal J_\beta(\pi;c)
=
\mathbb E_{\tau=(z_T,\ldots,z_0)\sim p_\pi(\cdot\mid c)}
\left[
R(z_0;c)
-
\beta
\sum_{h=0}^{T-1}
\mathrm{KL}\!\left(
\pi(\cdot\mid s_h)
\,\|\,
\pi_{\mathrm{ref}}(\cdot\mid s_h)
\right)
\right],
\quad
\beta>0.
\end{equation}
This objective is a population-level analytical reference, while PPO approximates the corresponding policy improvement in practice using sampled rollouts and a learned critic.

\begin{proposition}[Regularized Bellman characterization]
\label{prop:soft_bellman}
For a fixed condition \(c\), assume policy \(\pi\) is supported on the support of \(\pi_{\mathrm{ref}}\), ensuring the KL term is finite.
Let \(V_h^\star(s)\) denote the optimal KL-regularized value from MDP step \(h\).
At the terminal step, for \(s_T=(z_0,0,c)\),
$
V_T^\star(s_T)=R(z_0;c).
$
For \(h=T-1,\dots,0\), define
\begin{align}
Q_h^\star(s,a)
&=
V_{h+1}^\star(s'),
\qquad
s'=(a,T-h-1,c),
\\
V_h^\star(s)
&=
\beta
\log
\sum_{a\in\mathcal A(s)}
\pi_{\mathrm{ref}}(a\mid s)
\exp\!\left(
\frac{Q_h^\star(s,a)}{\beta}
\right).
\end{align}
Then the maximizer of \cref{eq:regularized_control_objective} is
\begin{equation}
\label{eq:soft_opt_policy}
\pi_h^\star(a\mid s)
=
\pi_{\mathrm{ref}}(a\mid s)
\exp\!\left(
\frac{Q_h^\star(s,a)-V_h^\star(s)}{\beta}
\right),
\end{equation}
with support contained in that of the reference policy, i.e.,
\(\pi_h^\star(a\mid s)=0\) whenever \(\pi_{\mathrm{ref}}(a\mid s)=0\).
The optimal value is
\[
\mathbb E_{z_T\sim p(\cdot)}
\left[
V_0^\star((z_T,T,c))
\right].
\]
\end{proposition}
The optimal policy is an exponential reweighting of the reference reverse kernel toward actions with larger downstream value \(Q_h^\star\).
For any two actions \(a,a'\) with positive reference probability,
\begin{equation}
\label{eq:relative_action_tilt}
\frac{
\pi_h^\star(a\mid s)/\pi_{\mathrm{ref}}(a\mid s)
}{
\pi_h^\star(a'\mid s)/\pi_{\mathrm{ref}}(a'\mid s)
}
=
\exp\!\left(
\frac{
Q_h^\star(s,a)-Q_h^\star(s,a')
}{\beta}
\right).
\end{equation}
Thus, actions with higher downstream value are amplified relative to \(\pi_{\mathrm{ref}}\).
The proof follows by applying a standard Gibbs variational identity at each state, details are given in Appendix~\ref{app:gibbs_identity}.

\begin{corollary}[Downstream-value amplification]
\label{cor:bellman_amplification}
Fix a state \(s\), a reverse step \(h\), and a subset of actions \(\mathcal G\subseteq\mathcal A(s)\).
Suppose there exist \(b\in\mathbb R\) and \(\Delta>0\) such that
\[
Q_h^\star(s,a)\ge b+\Delta
\quad\text{for all }a\in\mathcal G,
\qquad
Q_h^\star(s,a)\le b
\quad\text{for all }a\notin\mathcal G.
\]
Let \(p_G:=\pi_{\mathrm{ref}}(\mathcal G\mid s)\).
Then
\begin{equation}
\label{eq:bellman_good_action_mass}
\pi_h^\star(\mathcal G\mid s)
\ge
\frac{
e^{\Delta/\beta}p_G
}{
e^{\Delta/\beta}p_G+1-p_G
}.
\end{equation}
In particular, high-value actions with nonzero reference support receive amplified probability mass, while unsupported actions remain unsupported.
\end{corollary}

This provides a mechanism by which RL can improve target controllability: KL-regularized optimization amplifies high-value denoising choices while staying close to the SFT reference policy.
This amplification requires target-relevant actions to have nonzero support under \(\pi_{\mathrm{ref}}\). The readiness analysis in Appendix~\ref{app:experimental_sft_readiness} empirically probes this support through repeated sampling from the SFT checkpoint of \MODEL.

\subsubsection{Decision Complexity of Atom-level and Motif-aware RL}
\label{sec:theory_rl}

The reverse action \(a_h=z_{T-h-1}\) is a structured graph-valued object.
For a representation \(r\in\{\mathrm{atom},\mathrm{motif}\}\), write the one-step action as
$
a=(a_1,\dots,a_{M_r(s)}),
$
where \(M_r(s)\) is the number of stochastic categorical factors in the reverse kernel at state \(s\).
For motif graphs, these factors include motif types, inter-motif bonds, and directed attachment positions.
Assume the policy and reference policy factorize as
\begin{equation}
\label{eq:conditional_factorized_policy}
\pi^{(r)}(a\mid s)
=
\prod_{j=1}^{M_r(s)}
\pi^{(r)}_j(a_j\mid s,a_{<j}),
\qquad
\pi^{(r)}_{\mathrm{ref}}(a\mid s)
=
\prod_{j=1}^{M_r(s)}
\pi^{(r)}_{\mathrm{ref},j}(a_j\mid s,a_{<j}).
\end{equation}

\begin{proposition}[Factorized one-step KL]
\label{prop:factorized_kl}
Under \cref{eq:conditional_factorized_policy},
\begin{equation}
\label{eq:factorized_kl}
\mathrm{KL}\!\left(
\pi^{(r)}(\cdot\mid s)
\,\|\,
\pi^{(r)}_{\mathrm{ref}}(\cdot\mid s)
\right)
=
\sum_{j=1}^{M_r(s)}
\mathbb E_{a_{<j}\sim \pi^{(r)}}
\!\left[
\mathrm{KL}\!\left(
\pi^{(r)}_j(\cdot\mid s,a_{<j})
\,\|\,
\pi^{(r)}_{\mathrm{ref},j}(\cdot\mid s,a_{<j})
\right)
\right].
\end{equation}
\end{proposition}

\begin{corollary}[KL budget per subdecision]
\label{cor:kl_budget_dilution}
If
$\mathrm{KL}\!\left(
\pi^{(r)}(\cdot\mid s)
\,\|\,
\pi^{(r)}_{\mathrm{ref}}(\cdot\mid s)
\right)
\le \eta,
$
then
\begin{equation}
\label{eq:average_factor_kl}
\frac{1}{M_r(s)}
\sum_{j=1}^{M_r(s)}
\mathbb E_{a_{<j}\sim \pi^{(r)}}
\!\left[
\mathrm{KL}\!\left(
\pi^{(r)}_j(\cdot\mid s,a_{<j})
\,\|\,
\pi^{(r)}_{\mathrm{ref},j}(\cdot\mid s,a_{<j})
\right)
\right]
\le
\frac{\eta}{M_r(s)}.
\end{equation}
\end{corollary}
Therefore, under the same reference-KL regularization, a larger \(M_r(s)\) leaves less average policy change per categorical factor. Since atom-level representations induce larger \(M_r(s)\) through atom and atom-pair decisions, terminal-reward policy improvement becomes difficult.

We next compare this decision size under atom-only and motif-aware representations. For a molecule \(x\), let \(n_{\mathrm{atom}}(x)\) and \(n_{\mathrm{motif}}(x)\) denote its numbers of active atoms and active motifs. Since motifs partition the atom set, \(1\le n_{\mathrm{motif}}(x)\le n_{\mathrm{atom}}(x)\).
To compare one-step reverse decision sizes, define
\begin{equation}
\label{eq:action_count}
\begin{aligned}
L_{\mathrm{atom}}(n) &:= n+\binom{n}{2}=\frac{n^2+n}{2}, \qquad
L_{\mathrm{motif}}(n) &:= n+\binom{n}{2}+n(n-1)=\frac{3n^2-n}{2}.
\end{aligned}
\end{equation}
Here \(L_{\mathrm{atom}}\) counts atom-type and atom-bond decisions in a standard graph diffusion model, while \(L_{\mathrm{motif}}\) counts motif-type, inter-motif bond, and directed attachment-position decisions in the motif graph. 

\begin{corollary}[Motif-aware reduction of reverse decision size]
\label{cor:motif_factor_reduction}
Let \(\chi(x):=n_{\mathrm{atom}}(x)/n_{\mathrm{motif}}(x)\) be the atom-to-motif compression ratio. Then
\begin{equation}
\label{eq:motif_atom_count_ratio}
\frac{
L_{\mathrm{motif}}(n_{\mathrm{motif}}(x))
}{
L_{\mathrm{atom}}(n_{\mathrm{atom}}(x))
}
=
\frac{
3n_{\mathrm{motif}}(x)^2-n_{\mathrm{motif}}(x)
}{
n_{\mathrm{atom}}(x)^2+n_{\mathrm{atom}}(x)
}
\le
\frac{3}{\chi(x)^2}.
\end{equation}
Thus, sufficiently large motif compression yields a substantially smaller one-step reverse decision size despite the additional attachment-position channel.
\end{corollary}
Combining Corollary~\ref{cor:kl_budget_dilution} and Corollary~\ref{cor:motif_factor_reduction}, motif-aware actions improve the KL-regularized optimization geometry by reducing the number of categorical decisions that must be adjusted at each denoising step. 
In our implementation, the tokenizer yields an average atom-to-motif compression ratio above \(5.5\times\) for polymers (see Appendix~\ref{app:motif_tokenizer}), implying an upper-bound factor-count ratio of roughly \(3/5.5^2\approx 0.10\). 
The same tokenizer configuration yields weaker compression on small molecules, which is consistent with a relatively modest performance gain on FreeSolv in \cref{tab:main_results_mol}.

\section{Experiments}
\label{sec:experiment}
\textbf{RQ1}: We validate the controllable generative power of \MODEL against baselines from molecular optimization, generative modeling, and RL-based graph generation in \cref{sec:rq1}.
\textbf{RQ2}: We conduct further analysis to examine \MODEL in \cref{sec:rq2}.

\subsection{Experimental Setup}
\label{sec:exp-setup}
We evaluate \MODEL on three heterogeneous drug and material design task sets, covering both numerical and categorical properties.
For each task set, \MODEL is trained jointly across all tasks.
We also include \MODELSFT{}, an ablated variant without RL post-training, to isolate the effect of RL.
We assess performance across seven metrics covering validity, distribution learning, and condition control. 
Full dataset statistics and implementation details are provided in Appendix~\ref{app:dataset} and Appendix~\ref{app:experiments}.

\paragraph{Datasets}
For PT, we construct two unlabeled datasets: 13k real polymers for materials modeling and 10k molecules sampled from MoleculeNet~\citep{wu2018moleculenet} for drug-related tasks. 
For SFT, RL, and evaluation, we consider three task sets: (1) numerical polymer gas permeability conditions (O$_2$Perm, CO$_2$Perm, N$_2$Perm)~\citep{gaspolymer2012}; (2) numerical polymer density-functional-theory (DFT) properties (Eea, Egb, Egc)~\citep{Xu2023TransPolymer}; (3) a drug-related task set spanning physical chemistry (FreeSolv regression), biophysics (BACE classification), and physiology (BBBP classification)~\citep{molecule_opt_2022}. 
Two additional DFT properties, Ei and EPS, are held out from training for unseen-target generalization.
We separate gas permeability and DFT properties for polymers since they represent distinct property families. A six-target polymer joint-training analysis is provided in Appendix~\ref{app:polymer_joint_training}.

\paragraph{Evaluation}
We use an 8:1:1 train/validation/test split and evaluate conditional generation on held-out test conditions. 
For each setting, we generate 10,000 samples and report: (1) molecular validity (Validity); (2) internal diversity among the generated examples (Diversity); (3) fragment-based similarity with the reference set (Similarity); (4) Fr\'echet ChemNet Distance with the reference set (Distance)~\citep{frechetdis}; (5)-(7) MAE/Accuracy for the numerical/categorical task conditions (Property). 
Lower MAE or higher accuracy indicates stronger model controllability. 
We use Random Forests oracles for drug tasks~\citep{molecule_opt_2022} and GRIN for polymer tasks~\citep{grin2025}. Oracle analysis is provided in Appendix~\ref{app:experimental_oracle_selections}. We additionally report novelty and uniqueness in Appendix~\ref{app:exp_novelty_uniqueness}.

\paragraph{Baselines}
We compare against three families of baselines: 
(1) molecular optimization methods, including Graph-GA~\citep{jensen2019graphga}, MARS~\citep{xie2021mars}, JTVAE with Bayesian optimization (JTVAE-BO)~\citep{jtvae}, and SMILES-LSTM with hill climbing (LSTM-HC)~\citep{lstm}; 
(2) graph generative models, including GDSS~\citep{gdss}, DiGress~\citep{vignac2023digress}, MOOD~\citep{mood}, Graph DiT~\citep{liu2024graphdit}, and DeFoG~\citep{qin2025defog};  
(3) RL-based graph generation methods, including FREED~\citep{yang2021freed} and GDPO~\citep{liu2024gdpo}.
Baselines are trained in their standard task-specialized setting. For task-set-level distribution metrics, we report the strongest specialized baseline.

\subsection{RQ1: Heterogeneous Conditional Molecular Generation}
\label{sec:rq1}
\textbf{As shown in \cref{tab:main_results_gas,tab:main_results_dft,tab:main_results_mol},} \MODEL \textbf{achieves the best controllability on all nine targets and maintains at least 0.94 validity even without rule checking}.

\begin{table*}[t]
\centering
\small
\setlength{\tabcolsep}{5pt}
\caption{Heterogeneous Conditional Molecular Generation of 10K Polymers: Results on three numerical properties (gas permeability for O$_2$, N$_2$, CO$_2$). MAE is calculated between the input
conditions and the properties of the generated polymers using oracles. Best results are in \BEST{red}.}
\label{tab:main_results_gas}
\resizebox{\textwidth}{!}{%
\begin{tabular}{lcccccccc}
\toprule
\multirow{2}{*}{Model}
& \multicolumn{1}{c}{Validity $\uparrow$}
& \multicolumn{3}{c}{Distribution Learning}
& \multicolumn{4}{c}{Condition Control} \\
\cmidrule(lr){3-5} \cmidrule(lr){6-9}
& (w/o rule checking)
& Diversity $\uparrow$
& Similarity $\uparrow$
& Distance $\downarrow$
& O$_2$Perm $\downarrow$
& N$_2$Perm $\downarrow$
& CO$_2$Perm $\downarrow$
& Avg. MAE $\downarrow$ \\
\midrule
Graph GA             & 1.0000 (N.A.)   & 0.8838 & 0.9233 &  9.1974 &  1.9921 & 2.3010 & 1.9577 & 2.0836\\
MARS                 & 1.0000 (N.A.)   & 0.8302 & 0.9313 &  7.5231 &  1.8755 & 2.3397 & 1.6079 & 1.9410\\
LSTM-HC              & 0.9900 (N.A.)   & 0.8637 & 0.9492 & 6.3520  &  1.2168 & 1.5258 & 1.1370 &  1.2932\\
JTVAE-BO             & 1.0000 (N.A.)   & 0.7566 & 0.7524 & 24.6281 &  1.4501 & 1.2734 & 1.1529 & 1.2921\\
\midrule
DiGress              & 0.9820 (0.2543) & 0.8961 & 0.4350 & 26.0657 &  1.3147 & 1.5640 & 1.4411 &  1.4399  \\
DiGress v2           & 0.9788 (0.3927) & 0.9118 & 0.3027 & 18.9923 & 1.3355  & 1.5506 & 1.4427 &  1.4429 \\
GDSS                 & 0.9412 (0.8876) & 0.7829 & 0.0030 & 31.2573 & 1.4821  & 1.5844 & 1.3503 &  1.4723 \\
MOOD                 & 0.9713 (0.9004) & 0.8129 & 0.0924 & 35.8194 & 1.5727  & 1.6971 & 1.4646 & 1.5781\\
DeFoG                & 0.4609 (0.2988) & \BEST{0.9173} & 0.6828 & 15.1224 & 1.4314 & 1.5850 & 1.1998 & 1.4054 \\
Graph DiT            & 0.8176 (0.8536) & 0.8725 & 0.9624 & \BEST{5.8218} & 0.8185 & 0.9472 & 0.7836 & 0.8498 \\
\midrule
FREED                & 0.9466 (N.A.)   & 0.8284 & 0.3046 & 13.9356 & 0.7376 & 0.8423 & 0.7652 & 0.7817 \\
GDPO                 & 0.2578 (0.2965) & 0.8766 & 0.0973 & 28.4533 & 0.7934 & 0.8650 & 0.7527 & 0.8037\\
\midrule
\MODELSFT{} (ours)      & 0.9883 (0.9844) & 0.8564 & \BEST{0.9630}  &  9.4849 & 0.6442 & 0.7888 & 0.6709 & 0.7013 \\
\MODEL (ours)         & 0.9539 (0.9500) & 0.8644 & 0.9405 & 9.9854 & \BEST{0.3975} & \BEST{0.4720} & \BEST{0.4250} & \BEST{0.4315}\\
\bottomrule
\end{tabular}%
}
\end{table*}

\begin{table*}[t]
\centering
\small
\setlength{\tabcolsep}{5pt}
\caption{Heterogeneous Conditional Molecular Generation of 10K Polymers: Results on three numerical properties (Eea, Egb, Egc). MAE is calculated between the input
conditions and the generated properties. Best results are in \BEST{red}.}
\label{tab:main_results_dft}
\resizebox{\textwidth}{!}{%
\begin{tabular}{lcccccccc}
\toprule
\multirow{2}{*}{Model}
& \multicolumn{1}{c}{Validity $\uparrow$}
& \multicolumn{3}{c}{Distribution Learning}
& \multicolumn{4}{c}{Condition Control} \\
\cmidrule(lr){3-5} \cmidrule(lr){6-9}
& (w/o rule checking)
& Diversity $\uparrow$
& Similarity $\uparrow$
& Distance $\downarrow$
& Eea $\downarrow$
& Egb $\downarrow$
& Egc $\downarrow$
& Avg. MAE $\downarrow$ \\
\midrule
Graph GA    & 1.0000 (N.A.) & 0.8552 & 0.8569 & 7.9635 & 1.2773 & 1.4989 & 1.3846 &  1.3869 \\
MARS        & 1.0000 (N.A.) & 0.8617 & 0.8899 & 7.9519 &  1.3798 & 1.6541 & 1.4233 &  1.4857  \\
LSTM-HC     & 0.9976 (N.A.) & 0.8641 & 0.9325 & 5.2350 &  1.1667 & 1.3761 & 1.2062 &  1.2497  \\
JTVAE-BO    & 1.0000 (N.A.) & 0.7859 & 0.7520 & 11.7901 &  1.2574 & 1.4103 & 1.2900 &  1.3192  \\
\midrule
DiGress     & 0.6970 (0.4834) & 0.9112 & 0.4908 & 13.2764 & 0.8473 & 1.2600 & 0.7520 &  0.9531  \\
DiGress v2  & 0.7521 (0.5020) & 0.9055 & 0.5043 & 15.3710 & 0.8137 & 1.2511 & 0.7603 &  0.9417  \\
GDSS        & 0.6240 (0.6220) & \BEST{0.9208} & 0.0368 & 18.8047 & 1.2254 & 2.8129 & 1.8973 & 1.9785  \\
MOOD        & 0.7417 (0.6835) & 0.8246 & 0.2032 & 18.2511 & 1.6882 & 2.3201 & 1.7129 &  1.9071  \\
DeFoG       & 0.6504 (0.5452) & 0.8986 & 0.8033 & 6.0314 & 0.6714 & 0.9164 &  0.7453 & 0.7777 \\
Graph DiT   & 0.6868 (0.6720) & 0.8791 & 0.8965 & 5.4507 & 0.6075 & 0.5192  & 0.6739 &  0.6002  \\
\midrule
FREED       & 1.0000 (N.A.)  & 0.8308 & 0.2123 & 15.8095 & 0.5687 & 1.3398 & 0.9175  &  0.9420 \\
GDPO        & 0.6504 (0.7598) & 0.8902 & 0.8118 &  8.1592 & 0.5244 & 0.7081 & 0.5762 &  0.6029  \\
\midrule
\MODELSFT{} (ours)        & 0.9688 (0.9546) & 0.8860 & \BEST{0.9686} & \BEST{5.1352} & 0.2874 & 0.4713 & 0.3772 &  0.3786 \\
\MODEL (ours)        & 0.9431 (0.9414) & 0.8902 & 0.8527 & 9.3369 & \BEST{0.2392} & \BEST{0.4129} & \BEST{0.2813} & \BEST{0.3111}\\
\bottomrule
\end{tabular}%
}
\end{table*}

\begin{table*}[t]
\centering
\small
\setlength{\tabcolsep}{5pt}
\renewcommand{\arraystretch}{1.12}
\caption{Heterogeneous Conditional Molecular Generation of 10K Molecules: Results on one numerical property (FreeSolv) and two categorical properties (BACE, BBBP). MAE/Accuracy is calculated between the input
conditions and the generated properties. Avg. Rank is obtained by ranking each method per metric and averaging across property metrics. Best results are in \BEST{red}.}
\label{tab:main_results_mol}
\resizebox{\textwidth}{!}{%
\begin{tabular}{lcccccccc}
\toprule
\multirow{2}{*}{Model}
& \multicolumn{1}{c}{Validity $\uparrow$}
& \multicolumn{3}{c}{Distribution Learning}
& \multicolumn{4}{c}{Condition Control} \\
\cmidrule(lr){3-5} \cmidrule(lr){6-9}
& (w/o rule checking)
& Diversity $\uparrow$
& Similarity $\uparrow$
& Distance $\downarrow$
& FreeSolv $\downarrow$
& BACE $\uparrow$
& BBBP $\uparrow$
& Avg. Rank $\downarrow$ \\
\midrule
Graph GA   & 1.0000 (N.A.) & 0.8806 & \BEST{0.9658} & 7.6709 &   2.0405 & 0.4691 &  0.3017  & 12.3333\\
MARS       & 1.0000 (N.A.) & 0.8571 & 0.7914 & 8.9636 &   2.3988 & 0.5320 &  0.5389   & 9.0000\\
LSTM-HC    & 0.9802 (N.A.) & 0.9122 & 0.8590 & 14.8023 &  1.8267 & 0.5076 &  0.5242   & 8.0000\\
JTVAE-BO   & 1.0000 (N.A.) & 0.6996 & 0.6078 & 33.4123 &  1.9124 & 0.4632 &  0.4955   & 11.0000\\
\midrule
DiGress    & 0.5061 (0.4203) & 0.8961 & 0.6959 & 6.5017 & 1.9023 & 0.5263 & 0.6726  &  7.3333  \\
DiGress v2 & 0.5172 (0.3921) & 0.9016 & 0.7160 & 7.5507 & 1.8906 & 0.5215 & 0.6771  &  7.0000  \\
GDSS       & 0.7928 (0.8053) & 0.8968 & 0.5733 & 9.5036 & 3.5770 & 0.5036  & 0.5044 &  12.0000  \\
MOOD       & 0.7619 (0.5787) & 0.9126 & 0.1901 & 37.2665 & 2.0941 & 0.5066 & 0.4821 &  11.3333  \\
DeFoG      & 0.5034 (0.3188) & \BEST{0.9265} & 0.3836 & 8.1281 & 1.9216 & 0.5587 & 0.4636 &  9.3333 \\
Graph DiT  & 0.8344 (0.8217) & 0.8732 & 0.9147 & 9.1130 & 1.2390 & 0.9133  & 0.9342 &  2.6667  \\
\midrule
FREED       & 0.9980 (N.A.) & 0.8326 & 0.1018 & 20.6090 & 2.4163 & 0.7756 & 0.9213 & 7.3333  \\
GDPO       & 0.7307 (0.7266) & 0.8992 & 0.7126 & \BEST{5.3391} & 1.3598 & 0.8457 & 0.8556 &  4.3333 \\
\midrule
\MODELSFT{} (ours)        & 0.9668 (0.9473) & 0.8816 & 0.9616 & 7.6326 & 1.3169 & 0.9504 & 0.9439 &  2.3333  \\
\MODEL (ours)           & 0.9570 (0.9505) & 0.8422 & 0.6401 & 17.6755 & \BEST{1.0764} & \BEST{1.0000} & \BEST{0.9550} &  \BEST{1.0000} \\
\bottomrule
\end{tabular}%
}
\end{table*}

\paragraph{Chemical Validity}
Reported validity can overstate generative quality when it relies on hard-coded chemical rules. 
For example, graph-search methods such as GraphGA explicitly discard invalid molecules during mutation and crossover, thereby achieving perfect final validity. 
Removing rule checking in the decoding step, DiGress and MOOD suffer substantial validity drops (e.g., 0.98 to 0.25). 
Graph DiT is more robust, maintaining validity above 0.6, but still remains below rule-filtered search methods.
GDPO achieves only 0.2965 validity on the gas-permeability task set, further highlighting the fragility of atom-level RL.
In contrast, both \MODEL and \MODELSFT{} maintain validity above 0.9 across all three datasets without rule-based checking, suggesting that the motif-aware graph space preserves chemical feasibility under both SFT and RL alignment.

\paragraph{Distribution Learning}
GraphGA and Graph DiT are strong in-distribution baselines, whereas several models such as GDSS and MOOD struggle to match the reference distribution.  
\MODELSFT{} is highly competitive, achieving the best fragment similarity on both polymer tasks (0.9630 and 0.9686) and the best Fr\'echet distance on the DFT task set.
RL post-training shifts generation toward target-aligned regions, mildly on the polymer task sets but more substantially on the drug task set, where FreeSolv, BACE, and BBBP involve distinct physical-chemistry, target-bioactivity, and physiological-permeability objectives.
This reflects the general controllability-distribution trade-off: when structure-property rationales are weakly correlated across tasks, improved controllability may require larger deviations from the learned generative prior.

\paragraph{Condition Controllability}
Graph DiT is a strong conditional generation baseline, and RL-based methods such as GDPO and FREED are competitive because they directly optimize target-aligned rewards.
Nevertheless, \MODEL achieves the best performance across all targets, with an average MAE reduction of 27\% compared to its SFT-only variant \MODELSFT{}, which itself ranks second overall in controllability.
Compared with the best baseline in each setting, \MODEL reduces MAE by 44.8\% over FREED on gas-permeability properties and by 48.2\% over Graph DiT on DFT properties. 
On the drug-related task set, \MODEL reduces FreeSolv MAE by over 13\% and achieves a BACE accuracy of 1.0.

\subsection{RQ2: Ablation Studies and Model Analysis}
\label{sec:rq2}

\paragraph{Atom-level RL}
We replace the motif-aware graph representation in \MODEL and \MODELSFT{} with an atom-level representation on the gas-permeability task set while keeping all other settings unchanged. 
As shown in \cref{tab:rq2_motif_rl}, atom-level RL fails under \(w_{\mathrm{val}}=0.1\): rule-free validity drops from 0.95 to 0.07, fragment similarity drops from 0.94 to 0.18, and Avg.\ MAE increases from 0.43 to 1.57.
Increasing the validity weight to \(w_{\mathrm{val}}=0.5\) partially recovers validity but further worsens controllability, indicating that reward reweighting cannot resolve the bottleneck of atom-level RL.
The advantage of motif-aware representations is also visible before RL. These results support that atom-level RL is fragile under terminal 
rewards over long edit sequences, where stable structure-property 
relationships cannot be reliably learned.

\begin{table*}[t]
\centering
\small
\setlength{\tabcolsep}{5pt}
\caption{Atom-level RL ablation on the gas-permeability dataset.
\(w_{\mathrm{val}}\) is the validity-reward weight.
Atom-level RL remains unstable and poorly controllable even with a larger validity weight.}
\label{tab:rq2_motif_rl}
\resizebox{\textwidth}{!}{%
\begin{tabular}{lccccccccc}
\toprule
\multirow{2}{*}{Model}
& \multirow{2}{*}{\(w_{\mathrm{val}}\)}
& \multicolumn{1}{c}{Validity $\uparrow$}
& \multicolumn{3}{c}{Distribution Learning}
& \multicolumn{4}{c}{Condition Control} \\
\cmidrule(lr){4-6} \cmidrule(lr){7-10}
&
& (w/o rule checking)
& Diversity $\uparrow$
& Similarity $\uparrow$
& Distance $\downarrow$
& O$_2$Perm $\downarrow$
& N$_2$Perm $\downarrow$
& CO$_2$Perm $\downarrow$
& Avg. MAE $\downarrow$ \\
\midrule
\MODELSFT{}-atom    & -- & 0.6876 (0.6374) & 0.8902 & 0.8975 & 13.3321 & 1.0961 & 1.3941 & 1.0761 & 1.1888 \\
\MODELSFT{}         & -- & 0.9883 (0.9844) & 0.8564 & 0.9630 & 9.4849 & 0.6442 & 0.7888 & 0.6709  &  0.7013 \\
\midrule
\MODEL-atom       & 0.1 & 0.2266 (0.0742) & 0.8977 & 0.1849 & 32.9202 & 1.2423 & 1.8101 & 1.6688 & 1.5737 \\
\MODEL-atom       & 0.5 & 0.3164 (0.3086) & 0.9074 & 0.0454 & 32.4114 & 1.3512 & 1.8947 & 1.5299 & 1.6076 \\
\MODEL            & 0.1 & 0.9539 (0.9500) & 0.8644 & 0.9405 & 9.9854 & 0.3975 & 0.4720 & 0.4250 & 0.4315\\
\bottomrule
\end{tabular}%
}
\end{table*}

\paragraph{Generalization}
We further evaluate whether \MODEL can extend controllability to unseen property targets while keeping the generator frozen.
For two held-out DFT properties, Ei and EPS, we split each into a 
0.8/0.2 train/test split, using the training split to fit baselines 
and learn new task embeddings as convex combinations of the trained DFT task embeddings, e.g., $e_{\mathrm{Ei}}
=
0.68 e_{\mathrm{Eea}}
+
0.12 e_{\mathrm{Egb}}
+
0.20 e_{\mathrm{Egc}}.
$
As shown in \cref{tab:rq3_generalization}, \MODEL maintains high validity and achieves controllability competitive with baselines trained directly on these targets.
These results suggest that \MODEL captures structure-property knowledge 
that transfers across related targets, highlighting its potential for 
scalable and data-efficient extension to new properties.

\begin{table*}[t]
\centering
\small
\setlength{\tabcolsep}{5pt}
\caption{Generalization to unseen DFT targets (Ei, EPS). With frozen generator and task embeddings composed from existing tasks (Eea, Egb, Egc), \MODEL achieves competitive performance against task-specific baselines.}
\label{tab:rq3_generalization}
\resizebox{0.9\textwidth}{!}{%
\begin{tabular}{lccccccc}
\toprule
\multirow{2}{*}{Model}
& \multicolumn{1}{c}{Validity $\uparrow$}
& \multicolumn{3}{c}{Distribution Learning}
& \multicolumn{3}{c}{Condition Control} \\
\cmidrule(lr){3-5} \cmidrule(lr){6-8}
& (w/o rule checking)
& Diversity $\uparrow$
& Similarity $\uparrow$
& Distance $\downarrow$
& Ei $\downarrow$
& EPS $\downarrow$
& Avg. MAE $\downarrow$ \\
\midrule
LSTM-HC           & 0.9954 (N.A.) & 0.8294 & 0.9388 & 4.0644 &  1.0602 & 1.2078 &  1.1340    \\
DiGress           & 0.5460 (0.4630) & 0.9100 & 0.3590 & 15.3168 & 0.7400  & 0.8143 &  0.7772 \\
DeFoG           & 0.5624 (0.4948) & 0.8811 & 0.7238 & 7.4383 &  0.6921 & 0.6674 & 0.6768 \\
Graph DiT         & 0.6076 (0.7858) & 0.8654 & 0.9124 & 5.6182 & 0.4586 & 0.4868 & 0.4727  \\
\midrule
\MODEL (ours)      & 0.9182 (0.9037) & 0.8537 & 0.6420 & 15.5045 & 0.7231 & 0.7610 & 0.7421 \\
\bottomrule
\end{tabular}%
}
\end{table*}

\section{Related Work}
\label{sec:related}
\paragraph{Molecular Optimization and Inverse Design}
Traditional molecular optimization searches chemical space with genetic algorithms, Monte Carlo tree search, or Bayesian optimization, typically requiring repeated black-box oracle queries~\citep{jensen2019graphga,2015bayesian_opt,gao2022sample}.  
Recent graph generative models amortize this search by learning to sample molecules under target specifications, with diffusion and flow-based models such as GDSS, DiGress, Graph DiT, and DeFoG providing strong backbones for molecular graph generation~\citep{gdss,vignac2023digress,liu2024graphdit,qin2025defog}. 
In contrast, we target heterogeneous molecular inverse design and aim to bridge the controllability gap between existing methods and practical inverse-design needs.

\paragraph{Reinforcement Learning over Graph Generation}
Reinforcement learning provides a natural framework for aligning molecular generators with black-box objectives. 
Early methods optimize construction policies, such as fragment-based generation in FREED~\citep{yang2021freed}. 
Recent work such as GDPO studies policy-gradient optimization for graph diffusion and highlights the instability of directly optimizing atom-level reverse trajectories~\citep{liu2024gdpo}. 
Our method performs RL over motif-aware graph diffusion instead, using chemically meaningful actions to stabilize policy optimization while improving controllability.

\section{Conclusion}
\label{sec:conclusion}
In this work, we introduce \MODEL, controllable molecular generative foundation models for heterogeneous inverse design.
\MODEL turns pretrained graph-diffusion priors into controllable generation by optimizing conditional reverse policies via RL over motif-aware, chemically meaningful decisions.
We theoretically characterize the decision complexity of atom-level RL and justify motif-aware policy optimization.
Empirically, across three materials and drug-discovery benchmarks covering both numerical and categorical conditions, \MODEL ranks first in controllability on all nine targets, maintains validity above 0.94 without rule-based correction, and generalizes to unseen properties by updating only task embeddings.
These results suggest \MODEL as a practical foundation for unified, scalable, and controllable molecular inverse design across heterogeneous property tasks.

\section*{Acknowledgement}
This work was partially supported by NSF 2142827, 2146761, 2234058, 2332270, and 2341995 as well as AFOSR FA9550-25-1-0009. We also appreciate the support from the Foundation Models and Applications Lab at Notre Dame and ND-IBM Tech Ethics Lab.
We thank Gang Liu for valuable early discussions on reasoning over molecules, including the use of motif vocabularies in that context.
\bibliography{references}
\bibliographystyle{abbrvnat}


\appendix
\newpage
\section{Details on Model}
\label{app:additional_details}

\subsection{Model Architecture}
\label{app:model_architecture}

\paragraph{Motif-aware graph-level tokens}
We instantiate the denoiser as a motif-aware DiT-style transformer over padded motif-graph states
\(z_t=(X_t,E_t,P_t,m)\).
Following the graph-level token construction in DemoDiff~\citep{liu2025demodiff}, each transformer token corresponds to one motif slot and contains both its motif identity and its row-wise relation profile.
For an active motif slot \(i\), the initial token is
\[
u_i
=
g_X(X_{t,i})
\;\Vert\;
g_E(E_{t,i:})
\;\Vert\;
g_P(P_{t,i:}),
\]
where \(g_X\) is a motif-type embedding, and \(g_E,g_P\) are linear projections of the vectorized one-hot rows of inter-motif bond labels and directed attachment-position labels.
The fixed padded slot order is used when vectorizing \(E_{t,i:}\) and \(P_{t,i:}\), so the \(j\)-th block always corresponds to the relation from motif slot \(i\) to motif slot \(j\).
Thus, pair-specific topology is injected before self-attention rather than pooled into an unordered summary.
Null labels are used for non-bonded pairs and missing attachment positions, and entries involving padded slots are masked out.

\paragraph{Transformer backbone and conditioning}
The backbone is a stack of motif-slot self-attention blocks with DiT-style adaptive LayerNorm.
The condition consists of a task identity and a target value or label: the task identity is represented by a learned embedding, the target is encoded by a small MLP, and the two are fused into a condition vector.
This condition vector is added to the timestep embedding and used to modulate each transformer block.
For unconditional pretraining, the condition vector is omitted.
Since bond and attachment-position information is already included in the graph-level motif tokens, we do not add a separate edge bias to self-attention.

\paragraph{Output heads}
Let \(h_i\) be the final hidden representation of motif slot \(i\), split as
\[
h_i = h_i^X \Vert h_i^E \Vert h_i^P .
\]
The three parts predict motif identities, bond rows, and attachment-position rows through residual row-wise heads:
\[
\ell_i^X
=
W_X h_i^X+\mathrm{onehot}(X_{t,i}),
\]
\[
\widetilde L_i^E
=
\operatorname{reshape}(W_E h_i^E)+\mathrm{onehot}(E_{t,i:}),
\qquad
\ell_{ij}^E
=
\frac{1}{2}
\left(
\widetilde L_i^E[j]+\widetilde L_j^E[i]
\right),
\]
and
\[
L_i^P
=
\operatorname{reshape}(W_P h_i^P)+\mathrm{onehot}(P_{t,i:}),
\qquad
\ell_{ij}^P=L_i^P[j].
\]
The bond logits are symmetrized because \(E_{ij}=E_{ji}\), while the attachment-position logits remain directional because \(P_{ij}\) and \(P_{ji}\) encode different attachment positions.
All logits involving padded slots are masked out during training, sampling, and PPO log-probability computation.
The resulting logits parameterize the endpoint-state predictions \((\hat X_0,\hat E_0,\hat P_0)\), which are converted into reverse-step probabilities through the \(x_0\)-prediction posterior.
\subsection{Mask handling}
\label{app:mask_handling}

All main-text distributions are conditional on a fixed node mask \(m\). If \(m\) is not externally specified at sampling time, the full generative model can be written as
\begin{equation}
\label{eq:mask_prior_model}
p_\theta(z_{0:T},m\mid c)
=
\nu(m)\,p(z_T\mid m)\prod_{t=1}^{T}
\widetilde p_\theta(z_{t-1}\mid z_t,t,c,m),
\end{equation}
where \(\nu\) is a prior over masks (equivalently, over graph sizes). The main text suppresses this conditioning for readability.

\subsection{Masked denoising loss}
\label{app:masked_ce}
Let \(N_{\max}\) denote the fixed maximum number of motif nodes after padding.
For a mask \(m\), define the active node, undirected pair, and directed pair index sets
\[
\Omega_X(m)=\{i\in[N_{\max}] : m_i=1\},
\]
\[
\Omega_E(m)=\{(i,j):1\le i<j\le N_{\max},\; m_i m_j=1\},
\]
and
\[
\Omega_P(m)=\{(i,j):i\neq j,\; m_i m_j=1\}.
\]
Here \(\Omega_E(m)\) indexes undirected inter-motif bond variables, while
\(\Omega_P(m)\) indexes directional attachment-position variables. For directed
pairs without an attachment, \(P_{ij}\) is assigned a null attachment-position label.

Then
\begin{align}
\mathrm{CE}_X
&=
-\frac{1}{|\Omega_X(m)|}
\sum_{i\in\Omega_X(m)}
\sum_{a=1}^{d_X}
X_{0,ia}\log \hat X_{0,ia},
\\
\mathrm{CE}_E
&=
-\frac{1}{|\Omega_E(m)|}
\sum_{(i,j)\in\Omega_E(m)}
\sum_{b=1}^{d_E}
E_{0,ijb}\log \hat E_{0,ijb},
\\
\mathrm{CE}_P
&=
-\frac{1}{|\Omega_P(m)|}
\sum_{(i,j)\in\Omega_P(m)}
\sum_{p=1}^{d_P}
P_{0,ijp}\log \hat P_{0,ijp}.
\end{align}
If \(\Omega_E(m)\) or \(\Omega_P(m)\) is empty, the corresponding term is defined as \(0\).

\subsection{Factorized reverse kernel and one-step log-probability}
\label{app:factorized_logprob}

Conditioned on \(z_t\), \(t\), \(c\), and the fixed mask \(m\), we use a factorized categorical reverse kernel
\begin{align}
\widetilde p_\theta(z_{t-1}\mid z_t,t,c,m)
&=
\prod_{i\in\Omega_X(m)}
p_{\theta,t}^{X}\!\left(X_{t-1,i}\mid z_t,c,m\right)
\prod_{(i,j)\in\Omega_E(m)}
p_{\theta,t}^{E}\!\left(E_{t-1,ij}\mid z_t,c,m\right)
\nonumber\\
&\qquad\times
\prod_{(i,j)\in\Omega_P(m)}
p_{\theta,t}^{P}\!\left(P_{t-1,ij}\mid z_t,c,m\right).
\label{eq:factorized_reverse}
\end{align}
The \(P\)-factor is defined over all active directed motif pairs.
For pairs without an attachment, the corresponding categorical outcome is the null class \(0\).

Accordingly, for \(a_h=z_{t-1}\) with \(t=T-h\), the one-step log-probability used in PPO is
\begin{align}
\log \pi_\theta(a_h\mid s_h)
&=
\sum_{i\in\Omega_X(m)}
\log p_{\theta,t}^{X}\!\left(X_{t-1,i}\mid z_t,c,m\right)
+
\sum_{(i,j)\in\Omega_E(m)}
\log p_{\theta,t}^{E}\!\left(E_{t-1,ij}\mid z_t,c,m\right)
\nonumber\\
&\qquad+
\sum_{(i,j)\in\Omega_P(m)}
\log p_{\theta,t}^{P}\!\left(P_{t-1,ij}\mid z_t,c,m\right).
\label{eq:factorized_logprob_eq}
\end{align}
If the network predicts endpoint-state distributions \((\hat X_0,\hat E_0,\hat P_0)\), the factors above denote the corresponding \(z_{t-1}\)-probabilities induced by the reverse posterior parameterization.

\subsection{PPO optimization details}
\label{app:ppo_details}

\paragraph{Design rationale.}
The KL-regularized objective in \cref{eq:regularized_control_objective}
characterizes the desired policy improvement at the population level.
In translating this objective into a practical PPO implementation, we use a
standard surrogate decomposition: the critic is trained only against the
terminal task reward, while KL regularization toward the frozen SFT reference
policy is imposed as an explicit actor-side penalty.
Specifically, the Monte Carlo return used for value learning is
\[
G_h = R(z_0;c), \qquad h=0,\ldots,T-1.
\]
This keeps the value target tied to the fixed terminal reward function rather
than to a policy-dependent KL cost that would change as the actor is updated.
The KL coefficient \(c_{\mathrm{kl}}\) in the actor loss therefore plays the
empirical role of the regularization strength in the population objective,
although it is not in one-to-one correspondence with \(\beta\) because PPO uses
clipped importance ratios, finite on-policy batches, and normalized advantages.
Thus, the implemented PPO update is best interpreted as a tractable
regularized policy-improvement surrogate for \cref{eq:regularized_control_objective},
prioritizing stable value estimation and computational feasibility over exact dynamic-programming correspondence.

We initialize RL from the conditional diffusion model after SFT and denote the frozen reference policy by
\begin{equation}
\label{eq:pi0_def}
\pi_{\mathrm{ref}}:=\pi_{\theta_{\mathrm{ref}}}.
\end{equation}
Let \(\theta_{\mathrm{old}}\) denote the behavior-policy parameters used to collect the current on-policy batch.

With a value estimator \(V_\psi(s_h)\), the advantage is
\begin{equation}
\label{eq:advantage_def}
A_h
=
G_h-V_\psi(s_h).
\end{equation}

The PPO importance ratio is
\begin{equation}
\label{eq:ppo_ratio}
\rho_h
=
\frac{\pi_\theta(a_h\mid s_h)}{\pi_{\theta_{\mathrm{old}}}(a_h\mid s_h)}.
\end{equation}
The clipped actor loss, corresponding to the negative of the clipped surrogate in \cref{eq:ppo_clip_main}, is
\begin{equation}
\label{eq:ppo_clip}
\mathcal L_{\mathrm{clip}}
=
-\frac{1}{T}\sum_{h=0}^{T-1}
\mathbb E\!\left[
\min\!\Bigl(
\rho_h A_h,\;
\operatorname{clip}(\rho_h,1-\epsilon,1+\epsilon)A_h
\Bigr)
\right].
\end{equation}

The value loss, entropy bonus, and KL regularization toward the frozen reference policy are
\begin{equation}
\label{eq:ppo_value}
\mathcal L_{\mathrm{value}}
=
\frac{1}{T}\sum_{h=0}^{T-1}
\mathbb E\!\left[
\bigl(V_\psi(s_h)-G_h\bigr)^2
\right],
\end{equation}
\begin{equation}
\label{eq:ppo_ent}
\mathcal L_{\mathrm{ent}}
=
\frac{1}{T}\sum_{h=0}^{T-1}
\mathbb E\!\left[
\mathcal H\bigl(\pi_\theta(\cdot\mid s_h)\bigr)
\right],
\end{equation}
\begin{equation}
\label{eq:ppo_kl}
\mathcal L_{\mathrm{KL}}
=
\frac{1}{T}\sum_{h=0}^{T-1}
\mathbb E\!\left[
\mathrm{KL}\!\left(
\pi_\theta(\cdot\mid s_h)\,\|\,\pi_{\mathrm{ref}}(\cdot\mid s_h)
\right)
\right].
\end{equation}

The total PPO loss minimized in practice is
\begin{equation}
\label{eq:ppo_total_loss}
\mathcal L_{\mathrm{PPO}}
=
\mathcal L_{\mathrm{clip}}
+
c_v\,\mathcal L_{\mathrm{value}}
-
c_e\,\mathcal L_{\mathrm{ent}}
+
c_{\mathrm{kl}}\,\mathcal L_{\mathrm{KL}},
\end{equation}
where \(\epsilon>0\) is the clipping radius and \(c_v,c_e,c_{\mathrm{kl}}\ge 0\) are loss weights.
All expectations are taken over \(c\sim\mu_{\mathrm{RL}}\) and trajectories \(\tau\sim p_{\pi_{\theta_{\mathrm{old}}}}(\cdot\mid c)\).
In practice, advantages are normalized within each PPO batch.

\newpage
\section{Details on Theory}
\label{app:theory_proofs}

Throughout this appendix, we condition on a fixed target condition \(c\) and a fixed mask \(m\), with \(m\) suppressed in the main notation.
We denote by \(\pi_{\mathrm{ref}}\) the frozen reverse policy induced by the SFT checkpoint, which serves as the KL reference policy during PPO.
All state, action, and trajectory spaces are finite because \(N_{\max}\), \(d_X\), \(d_E\), \(d_P\), and \(T\) are finite.
If the mask is sampled, the same derivations apply after treating \(m\) as part of the state.
All policies are assumed to be absolutely continuous with respect to the reference policy at each state, i.e.,
\[
\pi(\cdot\mid s)\ll \pi_{\mathrm{ref}}(\cdot\mid s),
\]
so that the KL terms are well-defined.

\subsection{A Gibbs variational identity used in the Bellman proof}
\label{app:gibbs_identity}

\begin{lemma}[Gibbs variational identity]
\label{lem:gibbs_identity}
Let \(\mathcal A\) be a finite action set, let \(\pi_{\mathrm{ref}}\) be a probability distribution on \(\mathcal A\), let \(Q:\mathcal A\to\mathbb R\), and let \(\beta>0\). 
Then
\begin{equation}
\label{eq:gibbs_identity}
\sup_{\pi\ll \pi_{\mathrm{ref}}}
\left\{
\sum_{a\in\mathcal A}\pi(a)Q(a)
-
\beta\,\mathrm{KL}(\pi\,\|\,\pi_{\mathrm{ref}})
\right\}
=
\beta
\log
\sum_{a\in\mathcal A}
\pi_{\mathrm{ref}}(a)
\exp(Q(a)/\beta).
\end{equation}
The unique maximizer is
\begin{equation}
\label{eq:gibbs_policy}
\pi^\star(a)
=
\frac{
\pi_{\mathrm{ref}}(a)\exp(Q(a)/\beta)
}{
\sum_{a'\in\mathcal A}
\pi_{\mathrm{ref}}(a')\exp(Q(a')/\beta)
}.
\end{equation}
\end{lemma}

\begin{proof}
Define
\[
Z:=
\sum_{a\in\mathcal A}
\pi_{\mathrm{ref}}(a)\exp(Q(a)/\beta)
\]
and
\[
\pi^\star(a)
=
\frac{\pi_{\mathrm{ref}}(a)\exp(Q(a)/\beta)}{Z}.
\]
For any \(\pi\ll \pi_{\mathrm{ref}}\),
\begin{align*}
\mathrm{KL}(\pi\,\|\,\pi^\star)
&=
\sum_a
\pi(a)
\log
\frac{\pi(a)}{\pi^\star(a)}
\\
&=
\sum_a
\pi(a)
\log
\frac{\pi(a)Z}{\pi_{\mathrm{ref}}(a)\exp(Q(a)/\beta)}
\\
&=
\log Z
-
\frac{1}{\beta}
\sum_a \pi(a)Q(a)
+
\mathrm{KL}(\pi\,\|\,\pi_{\mathrm{ref}}).
\end{align*}
Rearranging gives
\[
\sum_a \pi(a)Q(a)
-
\beta\,\mathrm{KL}(\pi\,\|\,\pi_{\mathrm{ref}})
=
\beta\log Z
-
\beta\,\mathrm{KL}(\pi\,\|\,\pi^\star).
\]
The right-hand side is maximized uniquely when \(\pi=\pi^\star\), which proves the claim.
\end{proof}

\subsection{Proof of Proposition~\ref{prop:soft_bellman}}
\label{app:soft_bellman_proof}

\begin{proof}
We prove the result by backward induction.
At the terminal state \(s_T=(z_0,0,c)\), the value is
\[
V_T^\star(s_T)=R(z_0;c).
\]
Assume \(V_{h+1}^\star\) gives the optimal regularized value from step \(h+1\) onward.
At state \(s=s_h\), if the policy chooses a distribution \(\pi_h(\cdot\mid s)\), then the one-step regularized objective is
\[
\sum_{a\in\mathcal A(s)}
\pi_h(a\mid s)
Q_h^\star(s,a)
-
\beta\,
\mathrm{KL}\!\left(
\pi_h(\cdot\mid s)
\,\|\,
\pi_{\mathrm{ref}}(\cdot\mid s)
\right),
\]
where, because the MDP state update is deterministic,
\[
Q_h^\star(s,a)=V_{h+1}^\star(s'),
\qquad
s'=(a,T-h-1,c).
\]
Applying Lemma~\ref{lem:gibbs_identity} with \(Q(a)=Q_h^\star(s,a)\) yields
\[
V_h^\star(s)
=
\beta
\log
\sum_{a\in\mathcal A(s)}
\pi_{\mathrm{ref}}(a\mid s)
\exp(Q_h^\star(s,a)/\beta),
\]
and the unique maximizing policy is
\[
\pi_h^\star(a\mid s)
=
\pi_{\mathrm{ref}}(a\mid s)
\exp\!\left(
\frac{Q_h^\star(s,a)-V_h^\star(s)}{\beta}
\right).
\]
This proves the Bellman recursion and the policy form.
Taking expectation over the initial noise \(z_T\sim p(\cdot)\) gives the optimal value of \(\mathcal J_\beta(\pi;c)\).
\end{proof}

\subsection{Proof of Corollary~\ref{cor:bellman_amplification}}
\label{app:bellman_amplification_proof}

\begin{proof}
By \cref{eq:soft_opt_policy},
\[
\pi_h^\star(\mathcal G\mid s)
=
\frac{
\sum_{a\in\mathcal G}
\pi_{\mathrm{ref}}(a\mid s)\exp(Q_h^\star(s,a)/\beta)
}{
\sum_{a\in\mathcal A(s)}
\pi_{\mathrm{ref}}(a\mid s)\exp(Q_h^\star(s,a)/\beta)
}.
\]
Let \(p_G=\pi_{\mathrm{ref}}(\mathcal G\mid s)\).
Using the assumptions
\[
Q_h^\star(s,a)\ge b+\Delta
\quad (a\in\mathcal G),
\qquad
Q_h^\star(s,a)\le b
\quad (a\notin\mathcal G),
\]
we obtain
\[
\sum_{a\in\mathcal G}
\pi_{\mathrm{ref}}(a\mid s)\exp(Q_h^\star(s,a)/\beta)
\ge
e^{(b+\Delta)/\beta}p_G,
\]
and
\[
\sum_{a\notin\mathcal G}
\pi_{\mathrm{ref}}(a\mid s)\exp(Q_h^\star(s,a)/\beta)
\le
e^{b/\beta}(1-p_G).
\]
Therefore
\[
\pi_h^\star(\mathcal G\mid s)
\ge
\frac{
e^{(b+\Delta)/\beta}p_G
}{
e^{(b+\Delta)/\beta}p_G
+
e^{b/\beta}(1-p_G)
}
=
\frac{
e^{\Delta/\beta}p_G
}{
e^{\Delta/\beta}p_G+1-p_G
}.
\]
\end{proof}

\subsection{Proof of Proposition~\ref{prop:factorized_kl}}
\label{app:factorized_kl_proof}

\begin{proof}
For compactness, omit the representation superscript \(r\) and the state \(s\). 
Let
\[
\pi(a)=\prod_{j=1}^{M}\pi_j(a_j\mid a_{<j}),
\qquad
\pi_{\mathrm{ref}}(a)=\prod_{j=1}^{M}\pi_{\mathrm{ref},j}(a_j\mid a_{<j}).
\]
Then
\[
\log
\frac{\pi(a)}{\pi_{\mathrm{ref}}(a)}
=
\sum_{j=1}^{M}
\log
\frac{
\pi_j(a_j\mid a_{<j})
}{
\pi_{\mathrm{ref},j}(a_j\mid a_{<j})
}.
\]
Taking expectation under \(a\sim \pi\),
\begin{align*}
\mathrm{KL}(\pi\,\|\,\pi_{\mathrm{ref}})
&=
\mathbb E_{a\sim \pi}
\left[
\sum_{j=1}^{M}
\log
\frac{
\pi_j(a_j\mid a_{<j})
}{
\pi_{\mathrm{ref},j}(a_j\mid a_{<j})
}
\right]
\\
&=
\sum_{j=1}^{M}
\mathbb E_{a_{<j}\sim \pi}
\left[
\sum_{a_j}
\pi_j(a_j\mid a_{<j})
\log
\frac{
\pi_j(a_j\mid a_{<j})
}{
\pi_{\mathrm{ref},j}(a_j\mid a_{<j})
}
\right]
\\
&=
\sum_{j=1}^{M}
\mathbb E_{a_{<j}\sim \pi}
\left[
\mathrm{KL}\!\left(
\pi_j(\cdot\mid a_{<j})
\,\|\,
\pi_{\mathrm{ref},j}(\cdot\mid a_{<j})
\right)
\right].
\end{align*}
Restoring \(s\) and \(r\) gives \cref{eq:factorized_kl}.
\end{proof}

\subsection{Proof of Corollary~\ref{cor:kl_budget_dilution}}
\label{app:kl_budget_dilution_proof}

\begin{proof}
By Proposition~\ref{prop:factorized_kl},
\[
\mathrm{KL}\!\left(
\pi^{(r)}(\cdot\mid s)
\,\|\,
\pi^{(r)}_{\mathrm{ref}}(\cdot\mid s)
\right)
=
\sum_{j=1}^{M_r(s)}
D_j(s),
\]
where
\[
D_j(s)
:=
\mathbb E_{a_{<j}\sim \pi^{(r)}}
\left[
\mathrm{KL}\!\left(
\pi^{(r)}_j(\cdot\mid s,a_{<j})
\,\|\,
\pi^{(r)}_{\mathrm{ref},j}(\cdot\mid s,a_{<j})
\right)
\right].
\]
If the left-hand side is at most \(\eta\), then
\[
\frac{1}{M_r(s)}
\sum_{j=1}^{M_r(s)}
D_j(s)
\le
\frac{\eta}{M_r(s)}.
\]
\end{proof}

\subsection{Proof of Corollary~\ref{cor:motif_factor_reduction}}
\label{app:motif_ratio_proof}

\begin{proof}
By definition,
\[
L_{\mathrm{atom}}(n)
=
n+\binom{n}{2}
=
\frac{n^2+n}{2},
\]
and
\[
L_{\mathrm{motif}}(n)
=
n+\binom{n}{2}+n(n-1)
=
\frac{3n^2-n}{2}.
\]
Therefore
\[
\frac{
L_{\mathrm{motif}}(n_{\mathrm{motif}}(x))
}{
L_{\mathrm{atom}}(n_{\mathrm{atom}}(x))
}
=
\frac{
3n_{\mathrm{motif}}(x)^2-n_{\mathrm{motif}}(x)
}{
n_{\mathrm{atom}}(x)^2+n_{\mathrm{atom}}(x)
}.
\]
Let
\[
\chi(x)=\frac{n_{\mathrm{atom}}(x)}{n_{\mathrm{motif}}(x)}.
\]
Since
\[
3n_{\mathrm{motif}}(x)^2-n_{\mathrm{motif}}(x)
\le
3n_{\mathrm{motif}}(x)^2
\]
and
\[
n_{\mathrm{atom}}(x)^2+n_{\mathrm{atom}}(x)
\ge
n_{\mathrm{atom}}(x)^2,
\]
we obtain
\[
\frac{
L_{\mathrm{motif}}(n_{\mathrm{motif}}(x))
}{
L_{\mathrm{atom}}(n_{\mathrm{atom}}(x))
}
\le
\frac{
3n_{\mathrm{motif}}(x)^2
}{
n_{\mathrm{atom}}(x)^2
}
=
\frac{3}{\chi(x)^2}.
\]
\end{proof}

\newpage
\section{Details on Data}
\label{app:data_detials}
\subsection{Datasets}
\label{app:dataset}

For all random sampling, we set random seed to 42.
\subsubsection{Pretraining Datasets}
\label{app:pretrain_data}

For unconditional pretraining (PT), we curate two unlabeled datasets.
\begin{itemize}
    \item Polymer: The polymer corpus contains 12,792 unique real-world polymers.
    \item Small molecule: We curate a 10,000-molecule unlabeled corpus from MoleculeNet~\citep{wu2018moleculenet} by sampling valid and unique SMILES from datasets spanning physical chemistry, biophysics, and physiology.
\end{itemize}
No target labels or task identifiers are used during this stage, so PT remains fully unconditional. This design provides a broad structural prior before adapting the model to heterogeneous downstream molecular tasks.

We report RDKit-derived structural statistics in \cref{tab:pretrain_rdkit_stats}. 
We further visualize the real-atom and ring-count distributions in \cref{fig:pretrain_dataset}.

\begin{table}[t]
\centering
\small
\setlength{\tabcolsep}{5pt}
\caption{
Structural statistics of the unconditional pretraining datasets. 
Values are reported as mean / maximum. 
}
\label{tab:pretrain_rdkit_stats}
\begin{tabular}{lrrrrrr}
\toprule
Dataset 
& \#SMILES 
& Valid (\%) 
& Real atoms 
& Hetero atoms 
& Rings 
& Mol. weight \\
\midrule
Polymer  
& 12,792 
& 100.0 
& 30.7 / 80 
& 6.2 / 40 
& 3.2 / 15 
& 424.5 / 1282.5 \\
Molecule 
& 10,000 
& 100.0 
& 20.9 / 50 
& 5.7 / 27 
& 2.5 / 15 
& 303.8 / 1295.7 \\
\bottomrule
\end{tabular}
\end{table}

\begin{figure}[t]
\centering
\includegraphics[width=1\linewidth]{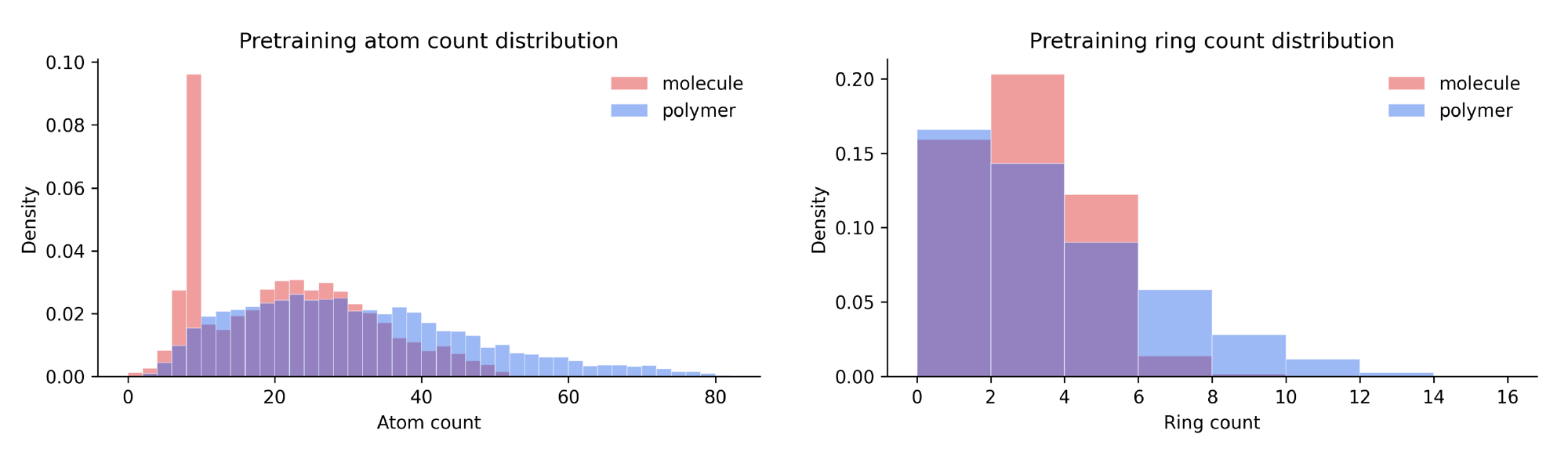}
\caption{Atom and ring count for pretraining datasets.}
\label{fig:pretrain_dataset}
\end{figure}

\subsubsection{Post-training Datasets}
\label{app:posttrain_data}
For SFT, RL, and evaluation, we construct three target-conditioned task sets, with statistics summarized in \cref{tab:posttrain_task_stats}.
The polymer DFT set contains electronic-structure properties relevant to functional polymer design, including electron affinity (Eea), bulk band gap (Egb), and chain band gap (Egc) for training, with ionization energy (Ei) and dielectric constant (EPS) held out as unseen targets~\citep{Xu2023TransPolymer}.
The polymer gas set includes O$_2$, N$_2$, and CO$_2$ permeability targets for membrane materials~\citep{gaspolymer2012}. All gas-permeability values are transformed into log space following previous work~\citep{ma2020pi1m}.
The drug set spans physical chemistry, target bioactivity, and physiological permeability through the numerical FreeSolv dataset, and two class-balanced datasets BACE and BBBP from~\citep{wu2018moleculenet}. 
\begin{table*}[t]
\centering
\small
\setlength{\tabcolsep}{4pt}
\caption{
Task definitions and raw statistics for post-training datasets.
Target ranges are reported in the original property units before any target transformation.
}
\label{tab:posttrain_task_stats}
\resizebox{0.9\textwidth}{!}{%
\begin{tabular}{llp{5.4cm}lrc}
\toprule
Benchmark & Task & Description & Type & \#SMILES & Target range \\
\midrule
Polymer DFT 
& Eea & DFT-computed electron affinity 
& Reg. & 368 & 0.394--5.18 \\
& Egb & DFT-computed bulk band gap 
& Reg. & 561 & 0.391--10.1 \\
& Egc & DFT-computed chain band gap 
& Reg. & 3,380 & 0.0205--9.86 \\
\midrule
Polymer gas 
& O$_2$Perm & Oxygen permeability of polymer membranes 
& Reg. & 443 & 0.005--$1.87{\times}10^4$ \\
& N$_2$Perm & Nitrogen permeability of polymer membranes 
& Reg. & 443 & 0.001--$1.66{\times}10^4$ \\
& CO$_2$Perm & Carbon dioxide permeability of polymer membranes 
& Reg. & 443 & 0.029--$4.70{\times}10^4$ \\
\midrule
Drug 
& FreeSolv & Hydration free energy of small molecules in water 
& Reg. & 642 & -25.5--3.43 \\
& BACE & Inhibition label for human $\beta$-secretase 1 
& Cls. & 1,332 & 0/1 \\
& BBBP & Blood--brain barrier permeability label 
& Cls. & 872 & 0/1 \\
\midrule
Generalization
& Ei & DFT-computed ionization energy 
& Reg. & 370 & 3.56--9.84 \\
& EPS & DFT-computed dielectric constant 
& Reg. & 382 & 2.61--9.09 \\
\bottomrule
\end{tabular}%
}
\end{table*}


We visualize task-set-level target distributions in \cref{fig:posttrain_dataset}. 
Regression tasks are shown as target-value histograms, while classification tasks are shown by class counts. 
\begin{figure}[t]
\centering
\includegraphics[width=1\linewidth]{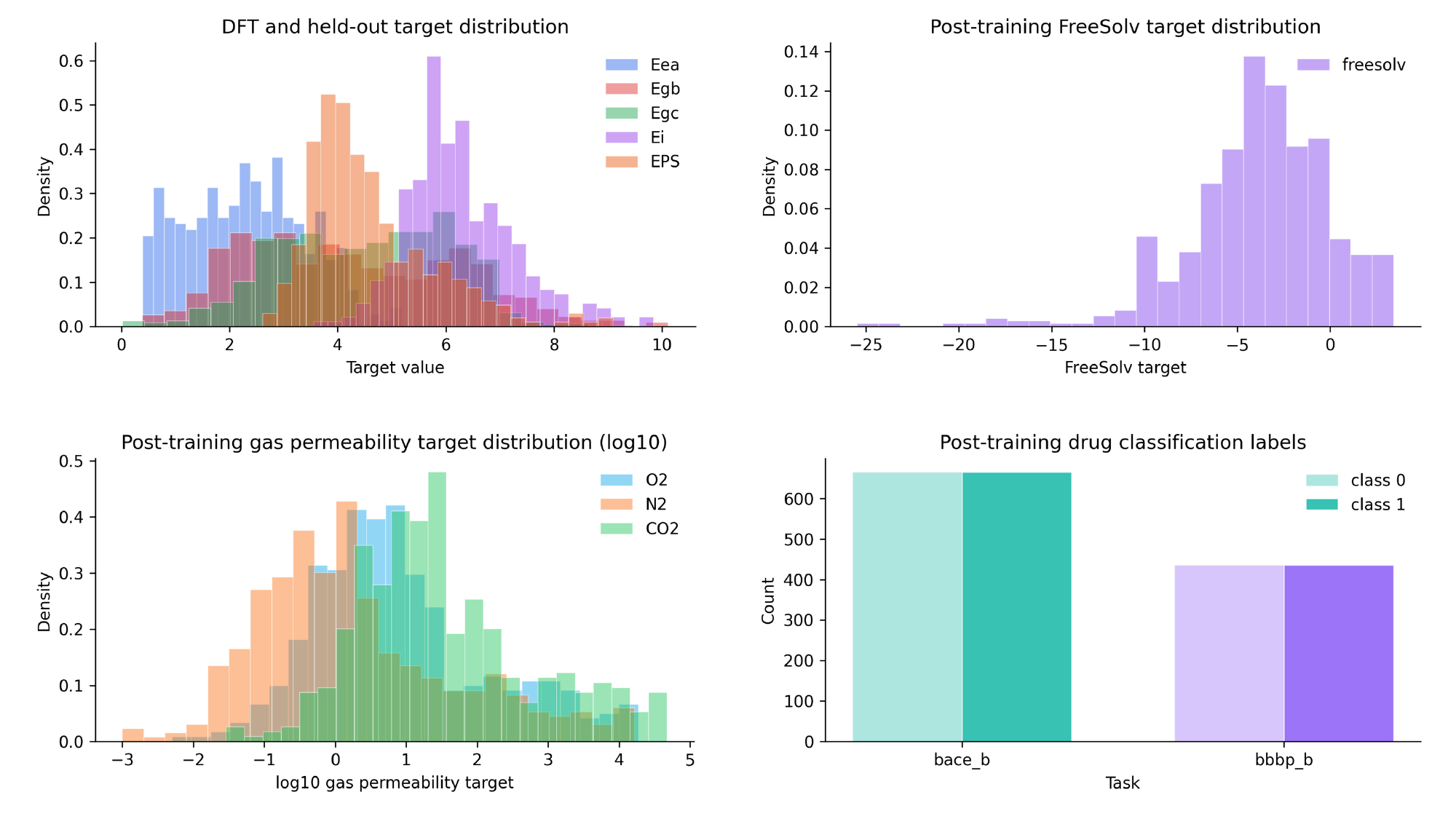}
\caption{Target distributions for conditional training and evaluation datasets.}
\label{fig:posttrain_dataset}
\end{figure}

For RL post-training, we use the same data sources as SFT. Within each task set, we randomly subsample each task’s training split to match the smallest training-set size among the tasks in that set.
The validation and test splits are kept fixed between SFT and RL, ensuring that both stages are evaluated on the same held-out conditions.

\subsection{Tokenizer Preparation}
We learn domain-specific tokenizers for the polymer and molecule domains using the same default configuration, V1000-R80, where V denotes the total motif vocabulary size and R denotes the ring vocabulary size.
\label{app:motif_tokenizer}

\subsubsection{Vocabulary Learning}
We construct the motif tokenizer using the data-driven Node Pair Encoding (NPE) algorithm of DemoDiff~\citep{liu2025demodiff}.
Each molecule is represented as a graph of atom-disjoint connected units that cover the original atom-level graph.
Inter-unit edges store the bond type and the attachment positions on both endpoint units, enabling lossless reconstruction to the original atom-level molecular graph.

Vocabulary learning starts from singleton atom tokens, including atom types observed in the pretraining corpus and the polymerization symbol ``*''.
For constrained NPE, we further enumerate frequent maximal ring units and add the top-\(R\) rings to the initial vocabulary, treating them as indivisible units during later merging to avoid ambiguous ring attachments.
NPE then repeatedly counts adjacent unit pairs in the tokenized training graphs, merges the most frequent pair into a new vocabulary unit, and retokenizes the graphs until the target vocabulary size is reached.

\subsubsection{Vocabulary Size Selection}
We choose the tokenizer configuration based on two criteria: motif-occurrence coverage and graph-size compression.
\paragraph{Motif-occurrence Coverage.} 
Motif-occurrence coverage measures the fraction of token occurrences explained by the top-\(k\) most frequent motifs after tokenization.
As shown in \cref{fig:tokenizer_coverage}, V1000-R80 provides the strongest coverage efficiency across polymer and molecule pretraining datasets, especially under small and medium retained-token budgets.
Larger vocabularies distribute occurrences over more rare motifs, reducing top-\(k\) coverage at the same retained vocabulary size.

\begin{figure}[t]
\centering
\includegraphics[width=\linewidth]{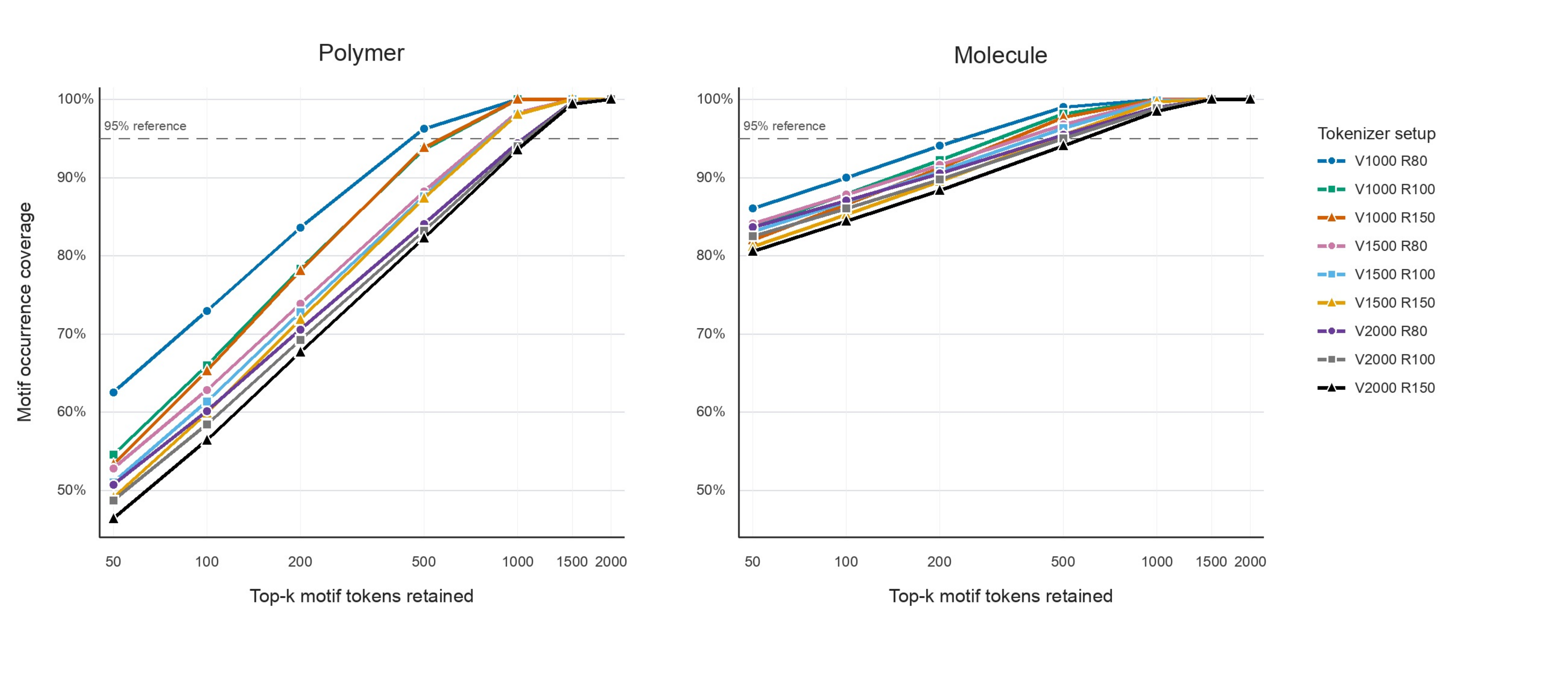}
\caption{
Motif-occurrence coverage under different tokenizer configurations.
Coverage is the fraction of token occurrences captured by the top-\(k\) most frequent motifs after tokenization.
V denotes the total motif vocabulary size, and R denotes the ring vocabulary size.
The dashed line marks the 95\% coverage reference.
}
\label{fig:tokenizer_coverage}
\end{figure}

\paragraph{Graph-size Compression.} 
We further evaluate graph-size compression in \cref{tab:node_compression_filtered}. 
V1000-R80 reduces polymer graphs from 32.71 atom nodes, including dummy wildcard attachment atoms, to 5.85 motif nodes on average, and molecule graphs from 20.87 atom nodes to 9.22 motif nodes.
This corresponds to an \(82.1\%\) node reduction for polymers and a \(55.8\%\) reduction for molecules.
Although larger vocabularies or ring vocabularies can further reduce graph size, V1000-R80 offers a better coverage-compression trade-off.
The weaker compression on small molecules is consistent with the more modest gains observed on the drug-related molecular task set in \cref{tab:main_results_mol}.

\begin{table}[t]
\centering
\small
\caption{
Average graph-size compression under representative tokenizer configurations.
Compression is the ratio of average atom nodes to average motif nodes, and reduction is the corresponding percentage decrease in node count.
}
\label{tab:node_compression_filtered}
\resizebox{0.8\textwidth}{!}{%
\begin{tabular}{llrrrr}
\toprule
Dataset & Tokenizer & Avg. atom nodes & Avg. motif nodes & Compression & Reduction \\
\midrule
\multirow{3}{*}{Polymer}
& V1000-R80  & 32.71 & 5.85 & 5.59$\times$ & 82.1\% \\
& V1000-R100 & 32.71 & 5.14 & 6.37$\times$ & 84.3\% \\
& V1500-R80  & 32.71 & 4.90 & 6.68$\times$ & 85.0\% \\
\midrule
\multirow{3}{*}{Molecule}
& V1000-R80  & 20.87 & 9.22 & 2.26$\times$ & 55.8\% \\
& V1000-R100 & 20.87 & 8.39 & 2.49$\times$ & 59.8\% \\
& V1500-R80  & 20.87 & 8.50 & 2.45$\times$ & 59.3\% \\
\bottomrule
\end{tabular}
}
\end{table}

We therefore use V1000-R80 as the default tokenizer configuration. 
It provides the best occurrence-coverage efficiency while still achieving substantial graph compression, offering a practical balance between compactness, motif expressiveness, and stable training.

\newpage
\section{Details on Experiments}
\label{app:experiments}

\paragraph{Hardware}
Each training or evaluation run can be executed on a single NVIDIA A100 GPU, no distributed training is required.

\subsection{Training Details}
\label{app:training_details}

All \MODEL variants use the same approximately 0.15B-parameter backbone, with transformer depth 12, hidden size 1024, 16 attention heads, and MLP ratio 4. 
We use $T = 500$ diffusion steps.
We train one \MODEL instance for each task set described in \cref{app:posttrain_data}.
Each instance follows the same three-stage pipeline: PT, SFT, and PPO-based RL.

We use a consistent set of training hyperparameters across task sets:
\begin{itemize}
    \item \textbf{PT} is run for 400 epochs with batch size 64, learning rate \(2\times10^{-5}\), gradient clipping 0.1, and weight decay \(10^{-12}\).
    \item \textbf{SFT} is run for up to 3,000 epochs with batch size 64 and learning rate \(2\times10^{-5}\).
    Checkpoints are selected by validation controllability averaged over tasks.
    \item \textbf{RL} is run for up to 400 epochs with batch size 32, learning rate \(1\times10^{-5}\), PPO clip range 0.2, value loss coefficient 0.5, entropy coefficient 0.001, KL coefficient 0.01, one rollout collection pass, and two PPO update passes.
    For computational efficiency, we apply partial-suffix fine-tuning to the final 30 reverse denoising steps, following~\citep{ren2024dppo}.
    For regression tasks, the property reward uses a Gaussian target-alignment term with \(\sigma=0.5\) after target normalization.
    For classification tasks, the property reward is the evaluator-predicted probability assigned to the target class.
    The terminal reward combines the property term with the validity term using \(w_{\mathrm{val}}=0.1\).
\end{itemize}

For baselines, we follow the official hyperparameter settings whenever available.
When official settings are unavailable or incomplete, we use Optuna~\citep{akiba2019optuna} to tune hyperparameters on the validation set and report test performance using the selected configuration.
All methods are evaluated on the same data splits, target conditions, and property evaluators.

\subsection{Computational Cost and Execution Time}
\label{app:runtime}

We report the wall-clock runtime and peak memory usage of the main training stages in \cref{tab:runtime}.
All runs are conducted on a single NVIDIA A100 GPU.
Peak CPU memory refers to host-side process memory, while peak GPU memory refers to allocated device memory.
The polymer pretraining run is shared by the two polymer benchmarks, namely Polymer gas and Polymer DFT, whereas the molecule pretraining run is used for the Drug benchmark.

The reported runtime and memory usage characterize the practical computational cost of training \MODEL under our experimental setting.
These measurements may vary with hardware configuration, software environment, I/O overhead, and cluster scheduling conditions.
\begin{table}[h]
\centering
\caption{Runtime and peak memory usage of the main training stages on a single NVIDIA A100 GPU.}
\label{tab:runtime}
\resizebox{0.9\textwidth}{!}{%
\begin{tabular}{lllll}
\toprule
\textbf{Stage} & \textbf{Data Type} & \textbf{Peak CPU Memory} & \textbf{Peak GPU Memory} & \textbf{Execution time} \\
\midrule
Pretraining & Polymer  & $\sim$5.8 GB & $\sim$12.9 GB & 3h 17m 20s \\
Pretraining & Molecule & $\sim$5.8 GB & $\sim$10.3 GB & 1h 45m 33s \\
\midrule
SFT & Polymer gas      & $\sim$8.0 GB & $\sim$19.4 GB & 6h 55m 45s \\
SFT & Polymer DFT      & $\sim$6.7 GB & $\sim$10.0 GB & 5h 35m 34s \\
SFT & Drug             & $\sim$6.1 GB & $\sim$10.3 GB & 4h 25m 1s \\
\midrule
RL post-training & Polymer gas & $\sim$6.0 GB & $\sim$7.4 GB & 2d 3h 49m \\
RL post-training & Polymer DFT & $\sim$6.0 GB & $\sim$7.5 GB & 1d 18h 41m \\
RL post-training & Drug        & $\sim$6.3 GB & $\sim$7.4 GB & 2d 7h 6m \\
\bottomrule
\end{tabular}
}
\end{table}

\subsection{Training Dynamics}
\label{app:training_dynamics}

We visualize validation dynamics during SFT and RL post-training in \cref{fig:dynamics_sft,fig:dynamics_rl}. 
For the drug benchmark, Avg.\ MAE is computed as the macro-average over FreeSolv, BACE, and BBBP, giving each task equal weight.
For categorical tasks such as BACE and BBBP, MAE is defined as the average absolute difference between the predicted positive-class probability and the binary target label.

During SFT, all three benchmarks exhibit stable training behavior.
Validity increases rapidly and remains high throughout training, while fragment similarity stays close to the reference distribution.
FCD decreases sharply in the early stage for the gas-permeability benchmark and remains relatively stable for the DFT and drug benchmarks.
Avg.\ MAE decreases substantially for the DFT and drug benchmarks and improves overall for the gas-permeability benchmark, despite moderate fluctuations.
These results indicate that SFT improves controllability while preserving the structural prior learned during pretraining.

RL post-training shows benchmark-dependent behavior. 
On the two polymer benchmarks, RL further improves Avg.\ MAE while largely maintaining distribution fidelity.
The gas-permeability benchmark preserves high validity and stable fragment similarity, whereas the DFT benchmark improves controllability with moderate fluctuations in validity and distribution metrics.

In contrast, the drug benchmark exhibits a stronger distribution shift during RL.
Although Avg.\ MAE decreases and validity remains high, FCD increases substantially and fragment similarity drops.
This trend is consistent with the heterogeneous nature of the drug benchmark, where FreeSolv, BACE, and BBBP span physical chemistry, target-specific bioactivity, and physiological permeability, respectively~\citep{wu2018moleculenet}.
Compared with a benchmark defined by a single drug-property family, this setting introduces more diverse target-alignment pressures, so RL improves controllability while shifting generation away from the aggregate reference distribution.
\begin{figure}[t]
\centering
\includegraphics[width=\linewidth]{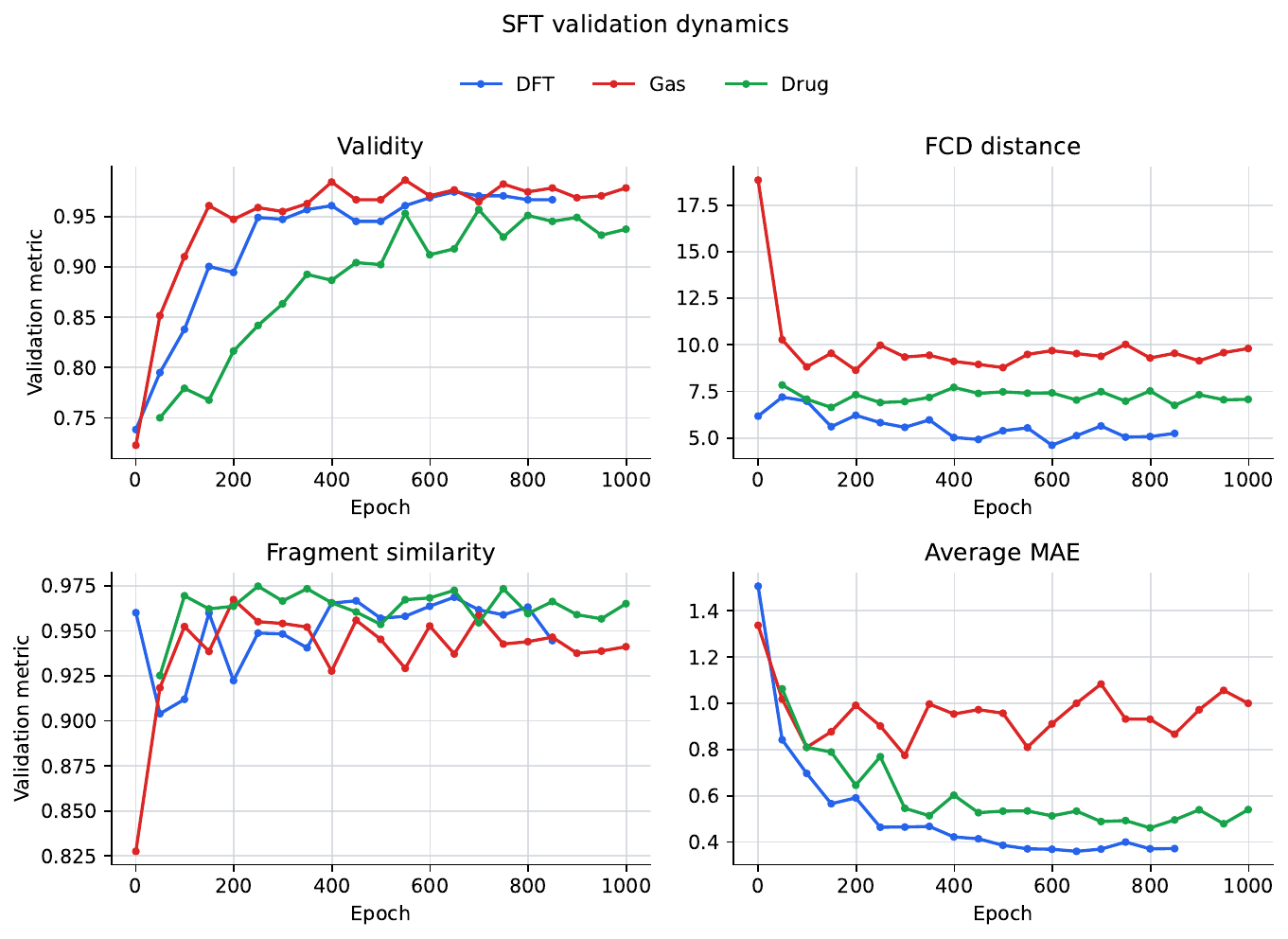}
\caption{
SFT validation dynamics across the polymer DFT, polymer gas-permeability, and drug benchmarks.
Validity denotes raw validity without rule-based checking.
SFT steadily improves controllability while maintaining high distribution fidelity.
}
\label{fig:dynamics_sft}
\end{figure}

\begin{figure}[t]
\centering
\includegraphics[width=\linewidth]{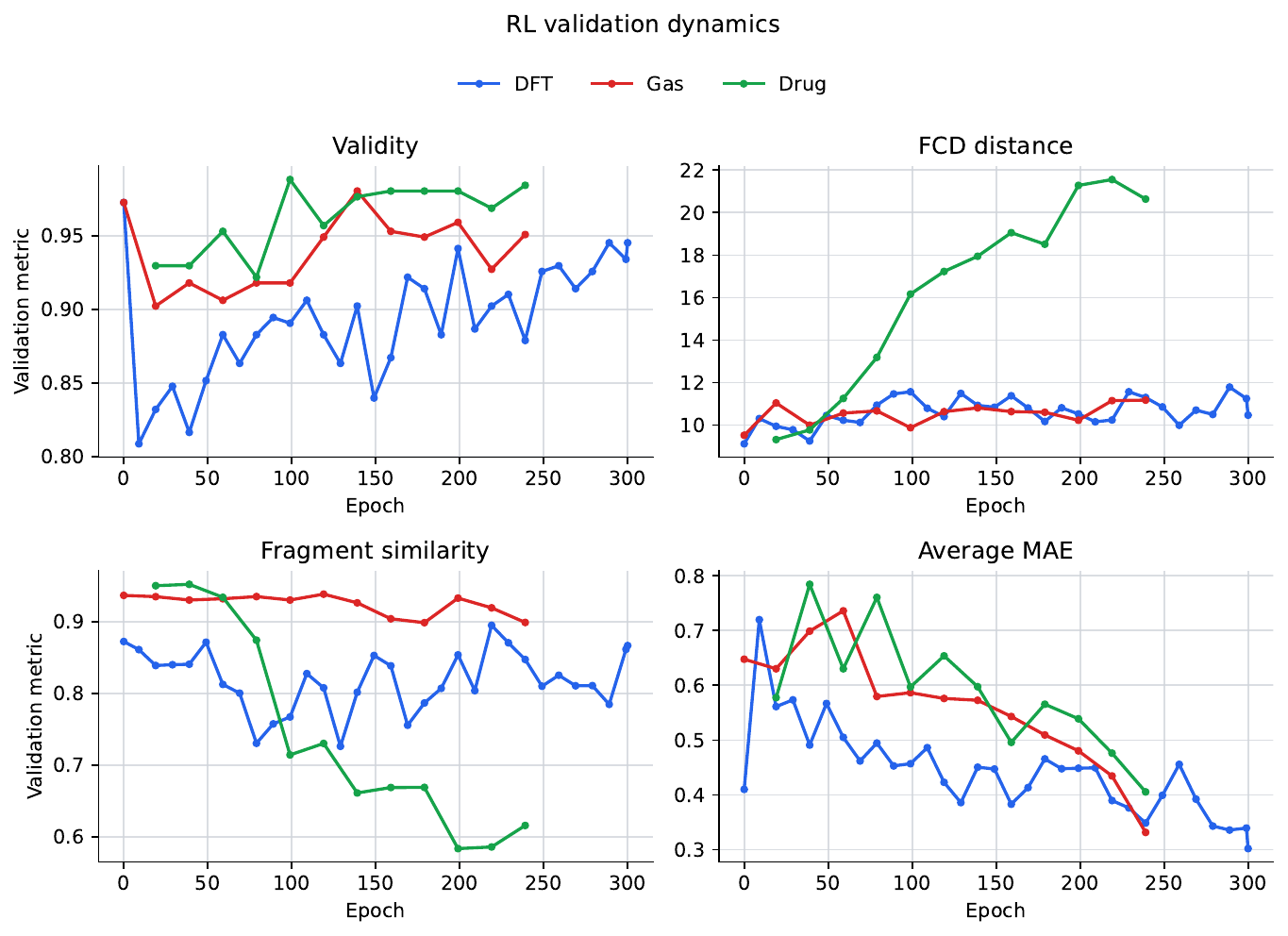}
\caption{
RL validation dynamics across the polymer DFT, polymer gas-permeability, and drug benchmarks. 
Validity denotes raw validity without rule-based checking.
RL further improves target controllability. 
The polymer benchmarks retain relatively stable distribution metrics, whereas the drug benchmark exhibits a stronger shift toward target-aligned regions.
}
\label{fig:dynamics_rl}
\end{figure}

\subsection{More Results}
\label{app:more_results}
In the following results, we report validity measured without rule-based 
checking, denoted as Validity$_{\mathrm{w/o\ check}}$.

\subsubsection{Joint Training on Polymer Properties}
\label{app:polymer_joint_training}
We report gas permeability and DFT properties as separate polymer task sets in the main experiments because they correspond to distinct property families and empirical distributions.
To examine whether they can be merged into a single heterogeneous polymer setting, we compare two training protocols:
\begin{itemize}
    \item \emph{Separated}: the main experimental setting, where gas-permeability and DFT targets are trained and evaluated using two benchmark-specific models.
    \item \emph{Joint}: a single model trained jointly on all six polymer targets, including O$_2$Perm, N$_2$Perm, CO$_2$Perm, Eea, Egb, and Egc.
\end{itemize}
For the separated setting, target MAEs are taken from the corresponding gas and DFT benchmark runs.
Validity$_{\mathrm{w/o\ check}}$ and distribution metrics are reported as the arithmetic mean of the two benchmark-level results, while Avg. MAE is the simple average over all six target-level MAEs.

As shown in \cref{tab:polymer_joint_training}, joint training does not cause a general collapse in validity or distribution learning.
However, merging the two polymer property families weakens target control.
At the SFT stage, the six-target Avg. MAE increases from 0.5400 in the separated setting to 0.5786 under joint training.
After RL, the degradation becomes more pronounced: Avg. MAE increases from 0.3713 to 0.5916.
This drop is mainly driven by gas-permeability targets, whose average MAE increases from 0.4315 to 0.8510, whereas DFT targets are less affected, changing from 0.3111 to 0.3322.
This suggests that gas-permeability targets introduce stronger distributional heterogeneity into the six-target setting.

These results indicate that merging heterogeneous polymer property families can introduce negative transfer in target controllability, even when validity and distribution-level metrics remain competitive.
We therefore keep gas-permeability and DFT properties as separate polymer task sets in the main evaluation.

\begin{table*}[t]
\centering
\small
\setlength{\tabcolsep}{4pt}
\caption{
Joint training analysis on polymer tasks.
\emph{Separated} denotes the main setting where gas permeability and DFT properties are trained as two benchmark-specific models.
\emph{Joint} denotes a single model trained jointly on all six polymer targets.
For \emph{Separated}, validity and distribution metrics are arithmetic means of the two benchmark-level results, while Avg. MAE is the simple average over the six target-level MAEs.
Best results are in \textbf{bold}.
}
\label{tab:polymer_joint_training}
\resizebox{\textwidth}{!}{%
\begin{tabular}{lccccccccccc}
\toprule
Model
& Validity$_{\mathrm{w/o\ check}}$ $\uparrow$
& Diversity $\uparrow$
& Similarity $\uparrow$
& Distance $\downarrow$
& O$_2$Perm $\downarrow$
& N$_2$Perm $\downarrow$
& CO$_2$Perm $\downarrow$
& Eea $\downarrow$
& Egb $\downarrow$
& Egc $\downarrow$
& Avg. MAE $\downarrow$ \\
\midrule
\MODELSFT{}-Joint
& 0.9609 & 0.8911 & \textbf{0.9743} & \textbf{3.8751}
& 0.6906 & 0.8394 & 0.7580
& 0.3525 & 0.4787 & 0.3522
& 0.5786 \\
\MODELSFT{}-Separated
& \textbf{0.9695} & 0.8712 & 0.9658 & 7.3101
& 0.6442 & 0.7888 & 0.6709
& 0.2874 & 0.4713 & 0.3772
& 0.5400 \\
\midrule
\MODEL{}-Joint
& 0.9277 & \textbf{0.8923} & 0.9524 & 4.6167
& 0.8245 & 1.0396 & 0.6889
& 0.2668 & 0.3522 & 0.3776
& 0.5916 \\
\MODEL{}-Separated
& 0.9457 & 0.8773 & 0.8966 & 9.6612
& \textbf{0.3975} & \textbf{0.4720} & \textbf{0.4250}
& \textbf{0.2392} & \textbf{0.4129} & \textbf{0.2813}
& \textbf{0.3713} \\
\bottomrule
\end{tabular}%
}
\end{table*}

\subsubsection{Other RL Post-training Objective}
\label{app:alternative_rl_objective}

Our main experiments instantiate RL post-training with PPO.
To test whether our pipeline is specific to PPO, we additionally evaluate a GDPO~\citep{liu2024gdpo} policy optimization objective on the gas-permeability polymer benchmark.
We denote this variant as \MODEL{}$_{\text{w/ GDPO}}$, which keeps the motif tokenizer, SFT checkpoint, terminal reward, and evaluation protocol unchanged.

As shown in \cref{tab:alternative_rl_objective_gas}, the GDPO update is compatible with our RL pipeline and improves target control over the SFT checkpoint.
Compared with \MODELSFT{}, \MODEL{}$_{\text{w/ GDPO}}$ reduces Avg. MAE from 0.7013 to 0.5879 while maintaining high raw validity.
It also substantially outperforms the original atom-level GDPO baseline in validity and distribution similarity, highlighting the benefit of motif-aware actions.

However, PPO remains stronger in this setting.
These results suggest that our pipeline is not tied to a specific RL algorithm, while PPO provides an effective instantiation for the main experiments.
\begin{table}[t]
\centering
\small
\setlength{\tabcolsep}{5pt}
\caption{
Alternative RL algorithm on the gas-permeability benchmark.
\MODEL{}$_{\text{w/ GDPO}}$ replaces PPO in \MODEL with a GDPO-style policy update, while keeping all other settings unchanged.
Best results are in \textbf{bold}.
}
\label{tab:alternative_rl_objective_gas}
\resizebox{\linewidth}{!}{%
\begin{tabular}{lcccccccc}
\toprule
Model
& Validity$_{\mathrm{w/o\ check}}$ $\uparrow$
& Diversity $\uparrow$
& Similarity $\uparrow$
& Distance $\downarrow$
& O$_2$Perm $\downarrow$
& N$_2$Perm $\downarrow$
& CO$_2$Perm $\downarrow$
& Avg. MAE $\downarrow$ \\
\midrule
GDPO                 
& 0.2965
& \textbf{0.8766}
& 0.0973 
& 28.4533 
& 0.7934 
& 0.8650 
& 0.7527 
& 0.8037 \\
\midrule
\MODELSFT{}      
& \textbf{0.9844}
& 0.8564 
& \textbf{0.9630}  
& \textbf{9.4849}
& 0.6442 
& 0.7888 
& 0.6709 
& 0.7013 \\
\MODEL{}$_{\text{w/ GDPO}}$
& 0.9286
& 0.7698 
& 0.8423 
& 13.9192
& 0.5842
& 0.6751
& 0.5044 
& 0.5879 \\
\MODEL{}          
& 0.9500 
& 0.8644 
& 0.9405 
& 9.9854 
& \textbf{0.3975} 
& \textbf{0.4720} 
& \textbf{0.4250} 
& \textbf{0.4315} \\
\bottomrule
\end{tabular}%
}

\end{table}

\subsubsection{Further Analysis of \MODELSFT{}}
\label{app:experimental_sft_readiness}

In this section, we analyze whether the SFT reference policy \(\pi_{\mathrm{ref}}\) mentioned in \cref{sec:theory_rl_improves} already assigns trajectory-level probability to low-error generations before RL.
Using the \MODELSFT{} trained on the polymer DFT benchmark, for each test condition \(c\), we draw \(K\) independent samples and compute the absolute target error under that condition.
We report two metrics:
\begin{itemize}
    \item \(\mathrm{mean\text{-}MAE}@K\) averages the errors over all sampled molecules and conditions, reflecting standard sampling quality.
    \item \(\mathrm{best\text{-}MAE}@K\) first selects, for each condition, the lowest-error sample among the \(K\) draws, and then averages this best-sample error over conditions.
\end{itemize}

As shown in \cref{tab:sft_readiness}, \(\mathrm{best\text{-}MAE}@K\) decreases from 0.3676 at \(K=1\) to 0.0261 at \(K=32\), while \(\mathrm{mean\text{-}MAE}@K\) remains relatively stable.
This suggests that the SFT reference policy already covers low-error candidates, but standard sampling does not select them reliably.
RL can therefore improve controllability by increasing the probability of these target-aligned generations.

\begin{table}[t]
\centering
\small
\setlength{\tabcolsep}{4pt}
\renewcommand{\arraystretch}{1.05}
\caption{Repeated sampling from \MODELSFT{} on the polymer DFT benchmark, with exactly $K$ draws per test condition.}
\label{tab:sft_readiness}
\resizebox{0.5\linewidth}{!}{%
\begin{tabular}{cccc}
\toprule
$K$ & Validity$_{\mathrm{w/o\ check}}$ $\uparrow$ & mean-MAE@K $\downarrow$ & best-MAE@K $\downarrow$ \\
\midrule
1  & 0.9815 & 0.3676 & 0.3676 \\
4  & 0.9647 & 0.3743 & 0.0808 \\
8  & 0.9641 & 0.3739 & 0.0476 \\
16 & 0.9676 & 0.3725 & 0.0322 \\
32 & 0.9671 & 0.3766 & \textbf{0.0261} \\
\bottomrule
\end{tabular}
}
\end{table}

\subsubsection{Oracle Robustness}
\label{app:experimental_oracle_selections}

To validate whether our conclusions depend on the property oracle used for reward and evaluation, we perform an oracle robustness analysis.
During RL post-training, the property component of the terminal reward is computed by a learned oracle, so a model could in principle overfit to the biases of a specific evaluator.

We therefore freeze the generated candidates from the main experiments and re-score them using alternative predictors trained on the same training splits, including Random Forest~\citep{randomforest}, Support Vector Machine, and GRIN~\citep{grin2025}.
For each oracle, we compute the corresponding controllability metric for every target, rank models by target-level performance, and report the average rank across the nine evaluated properties in \cref{tab:oracle_generation_evaluation}.

Although the ordering of some baselines varies across oracles, \MODEL consistently ranks first under all three evaluators. 
This suggests that our main conclusion is robust to the choice of learned oracle and is not driven by a specific evaluator. 

In the main experiments, we use the best-performing validation oracle for each domain: Random Forest for the drug-related molecular task set and GRIN for the polymer task sets.
Since GRIN can be trained as a general graph predictor, its main advantage comes from repetition-aware augmentation for polymer repeat-unit representations, which provides limited benefit for small molecules.

\begin{table}[t]
\centering
\small
\setlength{\tabcolsep}{6pt}
\caption{
Oracle robustness for generation evaluation.
Models are sorted by their average target-level rank across the nine evaluated properties, and only ranks 1--9 are shown here.
The ordering of baselines varies across evaluators, but \MODEL consistently ranks first.
}
\label{tab:oracle_generation_evaluation}
\begin{tabular}{clll}
\toprule
Avg. Rank & Random Forest & Support Vector Machine & GRIN\\
\midrule
1 & \textbf{\MODEL}     & \textbf{\MODEL}    & \textbf{\MODEL}\\
2 & Graph DiT  & FREED     & GDPO \\
3 & GDPO       & GDPO      & Graph DiT\\
4 & DiGress v2 & Graph DiT & DeFoG \\
5 & FREED      & LSTM-HC   & DiGress v2 \\
6 & DiGress    & JTVAE-BO  & FREED\\
7 & LSTM-HC    & DiGress   & DiGress\\
8 & MARS       & DiGress v2& LSTM-HC\\
9 & DeFoG      & DeFoG     & JTVAE-BO\\
\bottomrule
\end{tabular}
\end{table}

\subsubsection{Novelty and Uniqueness}
\label{app:exp_novelty_uniqueness}

We additionally report novelty and uniqueness as complementary generation-quality metrics.
These metrics are computed over generated samples aggregated across the evaluated target conditions.
Novelty measures the fraction of valid generated molecules absent from the training set, while uniqueness measures the fraction of non-duplicated valid generations after canonicalization.

For conditional generation, however, these metrics are mainly sanity checks against memorization and duplication.
They do not measure whether generated molecules satisfy the desired target properties or lie in a practically useful design region.
Prior work has shown that simple perturbations such as AddCarbon can achieve near-perfect scores without producing useful candidates~\citep{failuremode,tripp2023geneticalgorithmsstrongbaselines}.
Therefore, we treat novelty and uniqueness as complementary rather than primary indicators of inverse-design quality in our main experiments.

As shown in Table~\ref{tab:novelty_uniqueness}, \MODELSFT{} and \MODEL retain novelty and uniqueness, suggesting that \MODEL improves target controllability without collapsing to memorized or highly duplicated generations.

\begin{table}[t]
\centering
\caption{
Complementary novelty and uniqueness metrics across evaluated conditions.
Novelty and uniqueness are computed over valid generated molecules after canonicalization.
}
\label{tab:novelty_uniqueness}
\begin{tabular}{lcc}
\toprule
\textbf{Method} & \textbf{Novelty} & \textbf{Uniqueness} \\
\midrule
Graph GA   & 0.9950 & 1.0000 \\
MARS       & 1.0000 & 0.7500 \\
LSTM-HC    & 0.9421 & 0.9535 \\
JTVAE-BO   & 1.0000 & 0.6847 \\
DiGress    & 0.9908 & 0.7811 \\
DiGress v2 & 0.9801 & 0.7720 \\
GDSS       & 0.9136 & 0.2932 \\
MOOD       & 0.9867 & 0.9761 \\
DeFoG      & 0.9548 & 0.8819 \\
Graph DiT  & 0.9139 & 0.8739 \\
FREED      & 1.0000 & 1.0000 \\
GDPO       & 1.0000 & 0.9577 \\
\midrule
\MODELSFT{} & 0.9026 & 0.8817 \\
\MODEL      & 0.9390 & 0.8476 \\
\bottomrule
\end{tabular}
\end{table}

\subsection{Visualization}
Given the quantitative results in \cref{tab:main_results_gas,tab:main_results_dft}, we further visualize whether low-sample target-conditioned generation produces chemically plausible structures.
For selected O$_2$Perm, CO$_2$Perm, Eea, and Egb test conditions, we sample 10 candidates from \MODEL and show the rank-1 generated structure.
The rank is computed by averaging the candidate's rank in target-value error and its rank in structural similarity to the corresponding held-out reference structure.
The reference structure is used only for qualitative comparison and is not treated as the unique solution to the inverse-design problem.

As shown in \cref{fig:visualization_dft,fig:visualization_gas}, the selected generated structures often retain chemically meaningful motifs or scaffold patterns related to the held-out references, such as thiophene-containing units for Eea and Egb, aliphatic ester repeat units for high-Egb, aromatic imide motifs for CO$_2$Perm, and aromatic sulfone or bulky aromatic ester scaffolds for O$_2$Perm.
These examples qualitatively complement the quantitative results by showing that target-conditioned samples can remain close to plausible structure chemistry under limited sampling.

\begin{figure}[t]
    \centering
    \includegraphics[width=0.9\textwidth]{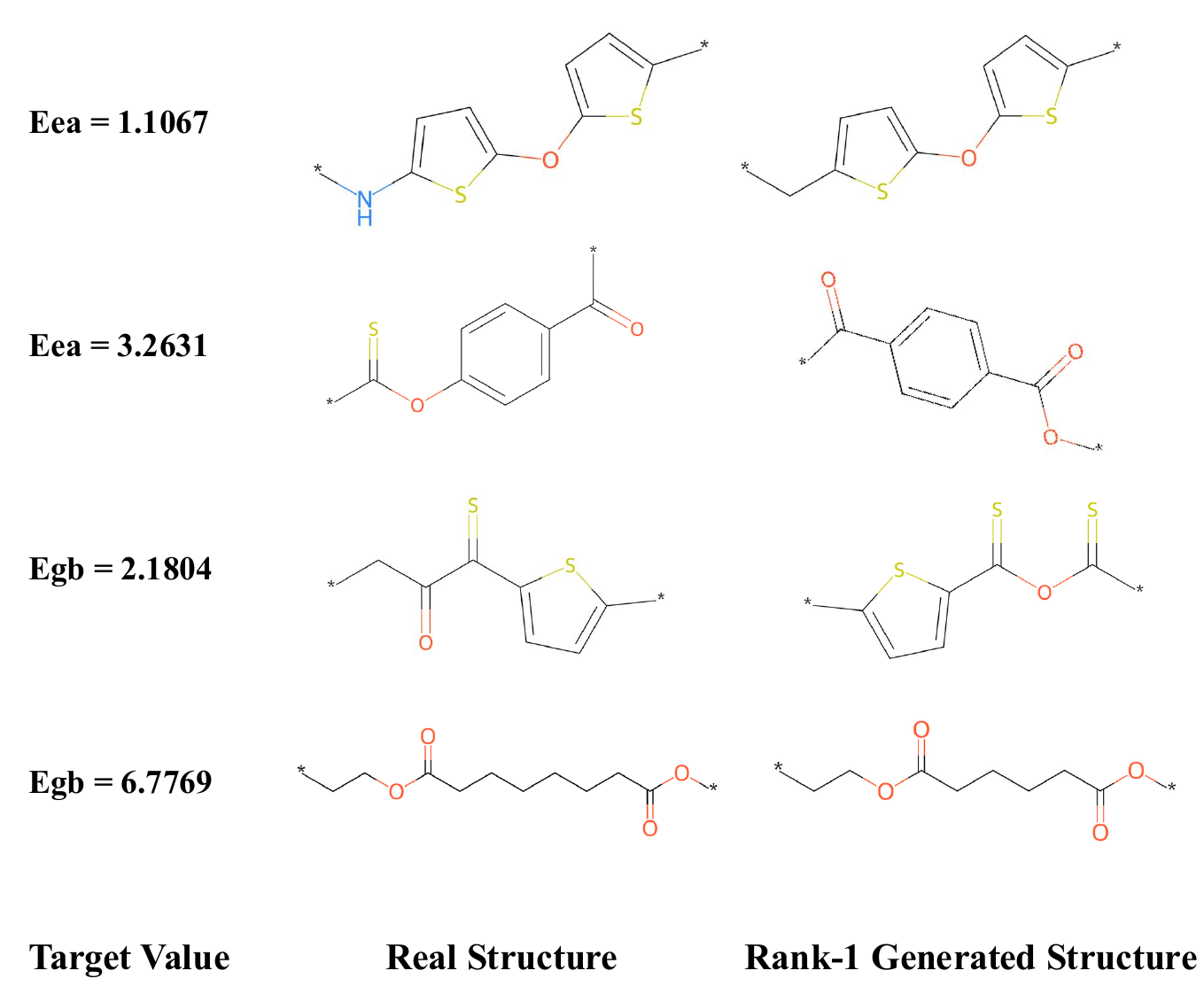}
    \caption{
    Rank-1 generated structures selected from 10 generated samples separately for Eea and Egb conditions.
    }
    \label{fig:visualization_dft}
\end{figure}
\begin{figure}[t]
    \centering
    \includegraphics[width=0.9\textwidth]{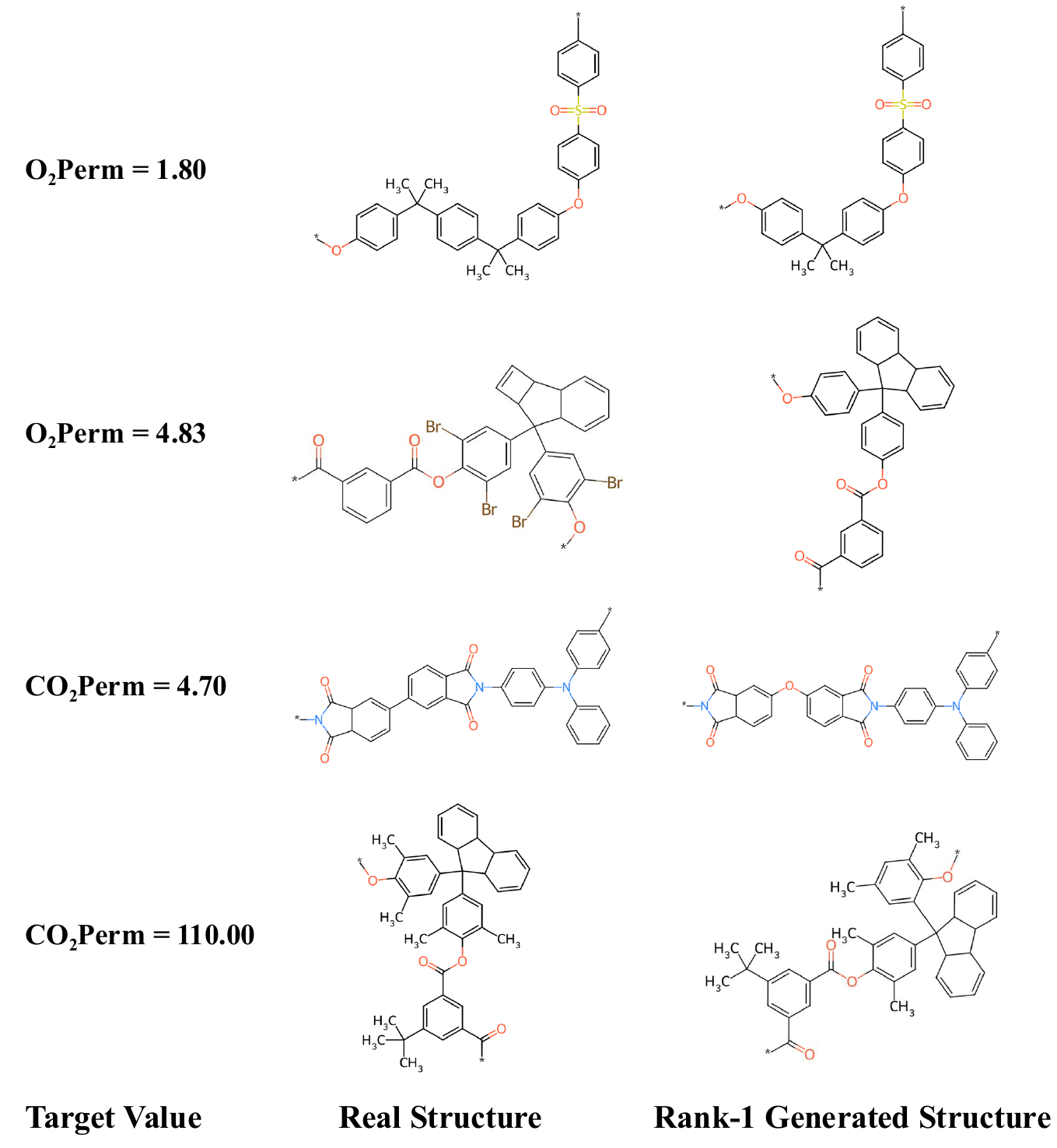}
    \caption{
    Rank-1 generated structures selected from 10 generated samples separately for O$_2$Perm and CO$_2$Perm conditions.
    }
    \label{fig:visualization_gas}
\end{figure}

\newpage
\section{Limitations and Future Directions}
\label{app:limitations}
\paragraph{Limitations}
First, our evaluation covers representative but not exhaustive chemical design settings.
For polymers, we focus on gas-permeability and DFT-derived electronic properties, while other practically important targets, such as thermal properties ($T_g$ and $T_c$) and mechanical properties, remain unexplored.
For small molecules, our benchmark includes regression and binary classification tasks, but does not yet cover broader condition types such as multi-class labels.
Expanding the task coverage would provide a more comprehensive assessment of heterogeneous controllable generation.

Second, we adopt a largely shared training configuration across benchmarks.
For example, \MODEL uses the same configuration for the polymer and molecule tokenizers, as described in \cref{app:motif_tokenizer}.
Although the V1000-R80 tokenizer provides strong compression for polymers, yielding an $82.1\%$ node reduction, the compression is less pronounced for small molecules.
In particular, the current tokenizer reduces the RL decision space and credit-assignment burden more substantially for polymers than for small molecules.

\paragraph{Future Directions}
Future work could develop stronger and more adaptive tokenizers to better balance compression, expressiveness, and controllability.
Promising directions include larger molecular vocabularies and unified tokenizers for polymers and small molecules that are less tied to a specific training corpus.

Another important direction is simultaneous multi-property control, where generated molecules or polymers must satisfy multiple numerical or categorical conditions at once.
This setting better reflects practical scientific design, where candidates are typically selected under coupled requirements involving property targets, chemical validity, synthesizability, and distributional plausibility.

More broadly, the long-term goal of controllable molecular generation is to advance scientific discovery by bridging the gap between hypothetical structure generation and practical inverse design.
Such models can help researchers prioritize more promising molecular candidates under heterogeneous property constraints before committing to expensive simulations, experiments, or screening cycles.
By combining transferable structural priors with task-conditioned adaptation and reward-guided post-training, \MODEL provides a step toward foundation-model-based design systems for molecules and polymers.



\end{document}